\documentclass{article}                     

\usepackage{amsmath,amssymb,amsthm,mathtools}
\usepackage[utf8]{inputenc}
\usepackage{xspace}
\usepackage{url}
\usepackage[numbers,sort&compress]{natbib}
\usepackage{dsfont}
\usepackage{enumerate}
\usepackage{booktabs}
\usepackage{multirow,bigdelim}

\usepackage{pdflscape}
\usepackage{multirow}
\usepackage{rotating}
\usepackage{subfigure}

\usepackage[margin=2.5cm]{geometry}

\usepackage{tikz}
\usetikzlibrary{calc}
\usetikzlibrary{arrows}
\usetikzlibrary{shapes}
\usepackage{pgfplots}
\usepackage{pgfplotstable}
\usepgfplotslibrary{statistics}

\usepgfplotslibrary{external}
\usetikzlibrary{pgfplots.external}
\tikzset{every mark/.append style={scale=0.90}}
\tikzstyle{plotcolor1}=[blue]
\tikzstyle{plotcolor2}=[red]
\tikzstyle{plotcolor3}=[brown!50!black]
\tikzstyle{plotcolor4}=[green!30!black]
\tikzstyle{plotcolor5}=[purple]
\tikzstyle{plotcolor6}=[orange]
\tikzstyle{plotcolor7}=[cyan]

\pgfplotsset{
cycle list={%
    {plotcolor1,mark=*},
    {plotcolor2,mark=square*},
    {plotcolor3,mark=pentagon*},
    {plotcolor4,mark=diamond*},
    {plotcolor5,mark=star},
    {plotcolor6,mark=x},
    {plotcolor7,mark=x},
    },
}

\usepackage[algo2e,ruled,vlined]{algorithm2e}
\SetAlgoSkip{}

\allowdisplaybreaks[3]

\newtheorem{theorem}{Theorem}

\newtheorem{lemma}[theorem]{Lemma}

\newtheorem{definition}{Definition}

\newcommand{\onemax}{\textsc{OneMax}\xspace}
\newcommand{\LeadingOnes}{\textsc{Leading\-Ones}\xspace}
\newcommand{\LO}{\LeadingOnes}

\newcommand{\hurdle}{\textsc{Hurdle}\xspace}

\newcommand{\zeros}[1]{|#1|_0}
\newcommand{\lEA}{\text{(1+$\lambda$)~EA}\xspace}
\newcommand{\EA}{\text{(1+1)~EA}\xspace}

\newcommand{\noise}{\mathrm{noise}}
\newcommand{\noisy}{\noise}
\newcommand{\muea}{($\mu$+1)~EA\xspace}

\newcommand{\transition}[2]{\ensuremath{\prob(#1 \to #2)}}
\newcommand{\equal}{\ensuremath{\mathit{Eq}}}

\DeclareMathOperator{\Prob}{Pr}
\DeclareMathOperator{\prob}{Pr}
\DeclareMathOperator{\E}{E}
\DeclareMathOperator{\poly}{poly}

\newcommand\Tstrut{\rule{0pt}{2.6ex}}         
\newcommand\Bstrut{\rule[-0.9ex]{0pt}{0pt}}   

\renewenvironment{proof}%
{\begin{trivlist}\item\textbf{Proof.}}%
{\hspace*{\fill}$\Box$\end{trivlist}}
\newenvironment{proofof}[1]
{\begin{trivlist}\item\textbf{Proof of #1.}}%
{\hspace*{\fill}$\Box$\end{trivlist}}

\begin{document}

\title{Analysing the Robustness of Evolutionary Algorithms to Noise:
Refined Runtime Bounds and an Example Where Noise is Beneficial}


\author{Dirk Sudholt\\University of Sheffield\\United Kingdom}


\maketitle

\begin{abstract}
We analyse the performance of well-known evolutionary algorithms \EA and \lEA in the prior noise model, where in each fitness evaluation the search point is altered before evaluation with probability~$p$.
We present refined results for the expected optimisation time of the \EA and the \lEA on the function \LO, where bits have to be optimised in sequence.
Previous work showed that the \EA on \LO runs in polynomial expected time if $p = O((\log n)/n^2)$ and needs superpolynomial expected time if $p = \omega((\log n)/n)$, leaving a huge gap for which no results were known. We close this gap by showing that the expected optimisation time is
$\Theta(n^2) \cdot \exp(\Theta(\min\{pn^2, n\}))$ for all $p \le 1/2$, allowing for the first time to locate the threshold between polynomial and superpolynomial expected times at $p = \Theta((\log n)/n^2)$. Hence the \EA on \LO is much more sensitive to noise than previously thought.
We also show
that offspring populations of size $\lambda \ge 3.42\log n$ can effectively deal with
much higher noise than known before.

Finally, we present an example of a rugged landscape where prior noise can help to escape from local optima by blurring the landscape and allowing a hill climber to see the underlying gradient. We prove that in this particular setting noise can have a highly beneficial effect on performance.
\end{abstract}


\section{Introduction}


Many real-world problems suffer from sources of uncertainty, such as noise in the fitness evaluation, changing constraints, or dynamic changes to the fitness function~\cite{Jin2005}. Evolutionary algorithms are well suited for dealing with these challenges due to their use of a population, and because they can often recover quickly from setbacks resulting from noise or dynamic changes.
They have proven to work well in many applications to combinatorial problems~\cite{Bianchi2009}.

However, our theoretical understanding of how evolutionary algorithms deal with noise is limited. It is often not clear how noise affects the performance of evolutionary algorithms, and how much noise an evolutionary algorithm can cope with.
For evolution strategies in continuous optimisation there exists
a rich body of work (see, e.\,g.\ \cite{BEYER2000239,Meyer-Nieberg2008,Jebalia2011} and the references therein), but there are only few rigorous theoretical analyses on the performance of noisy evolutionary optimisation in discrete spaces.

The first runtime analysis for discrete evolutionary algorithms in a noisy setting was given by~\citet{Droste2004} in the context of a simple algorithm called \EA on the well-known function \onemax$(x) := \sum_{i=1}^n x_i$, which simply counts the number of bits set to~1. He considered a setting now known as \emph{one-bit prior noise}, where with probability~$p$ a uniform random bit is flipped before evaluation. Hence, instead of returning the fitness of the evaluated search point, the fitness function may return the fitness of a random Hamming neighbour. He proved that, when $p = O((\log n)/n)$ the \EA can still optimise \onemax efficiently. But when $p = \omega((\log n)/n)$ the expected optimisation time becomes superpolynomial.

\citet*{Giessen2016} studied a more general class of algorithms, including the \EA, the \lEA that generates $\lambda$ new solutions (offspring) in parallel and picks the best one, and the \muea that keeps a population of $\mu$ search points. They considered prior noise and \emph{posterior noise}, where posterior noise means that noise is added to the fitness value, and presented an elegant approach that gives results in both noise models. They showed that the \EA on \onemax runs in expected time $O(n \log n)$ if $p = O(1/n)$, polynomial time if $p = O((\log n)/n)$, and superpolynomial time if $p = \omega((\log n)/n) \cap 1 - \omega((\log n)/n)$. The same results hold in the \emph{bit-wise noise} model, where each bit is flipped independently before evaluation with probability~$p/n$.
They also considered the function \LO that counts the length of the longest prefix that only contains bits set to~1.
For \LO they show a time bound of $O(n^2)$ if $p \le 1/(6en^2)$ and an exponential lower bound if $p = 1/2$.

The authors also found that using parent populations in a \muea can drastically improve robustness as survival selection removes one of the worst individuals, and a population increases the chances that a low-fitness individual will be correctly identified as having low fitness. Offspring populations also increase robustness as they amplify the probability that a clone of the current search point will be evaluated truthfully, thus lowering the chance of losing the best fitness.
For \LO they showed a time bound for the \lEA of $O(\lambda n + n^2)$ if $p \le 0.028/n$ and
$72 \log n \le \lambda = o(n)$. Note that their bound simplifies to $O(n^2)$ since $\lambda = o(n)$.

\citet{Dang2015} gave general results for prior and posterior noise in non-elitist evolutionary algorithms, that is, evolutionary algorithms where the best fitness in the population may decrease. The same authors~\cite{Dang2016} also considered noise resulting from only partially evaluating search points.

In terms of posterior noise, \citet{Sudholt2012} considered the performance of a simple ant colony optimiser (ACO) for computing shortest paths when path lengths are obscured by positive posterior noise modelling traffic delays.
They showed that noise can make the ants risk-seeking, tricking them onto a suboptimal path and leading to exponential optimisation times. \citet*{Doe-Hot-Koe:c:12} showed that this problem can be avoided if the parent is reevaluated in each iteration. \citet{Feldmann2013} further analysed the performance of fitness-proportional updates.
\citet*{Friedrich2017} showed that the compact Genetic Algorithm and ACO~\cite{Friedrich2015} are both efficient under extreme Gaussian posterior noise, while a simple \muea is not.

\citet*{Prugel-Bennett2015} considered a population-based algorithm using only selection and crossover, and showed that the algorithm can optimise \onemax with a large amount of noise.
\citet*{Qian2016} showed that noise can be handled efficiently by combining reevaluation and threshold selection. \citet*{Akimoto2015} as well as \citet*{Qian2016a} showed that resampling can essentially eliminate the effect of noise.

\begin{sidewaystable*}[phbt]
\centering%
\caption{Overview of results on the expected optimisation time on \LO with prior noise. Results for the \EA also hold for asymmetric one-bit noise, for which no results on \LO are available,
with the caveat that for $p = \omega(1/n)$ we only have an upper bound of
$O(n^2) \cdot e^{O(pn^2)}$.
The bound $O(\lambda n + n^2)$ from~\cite{Giessen2016} was simplified to $O(n^2)$ using their condition $\lambda = o(n)$.
}
\label{tab:overview}
\begin{tabular}{lll@{}l}
Setting & Previous work & & This work\\
\toprule
\multirow{5}{3cm}{\EA,\\one-bit noise~$p$} & $O(n^2)$ if $p \le 1/(6en^2)$~\cite[Cor.~18]{Giessen2016} &
\rdelim\}{13}{1em} & \multirow{13}{4.3cm}{$\Theta(n^2) \cdot e^{\Theta(\min\{pn^2, n\})}$ if $p \le 1/2$}\\
& $2^{\Omega(n)}$ if $p = 1/2$~\cite[Thm.~20]{Giessen2016} & &\\
& polynomial if $p = O((\log n)/n^2)$~\cite[Thm.~14]{Qian2018} & \\
 & superpolynomial if $p = \omega((\log n)/n) \cap o(1)$~\cite[Thm.~14]{Qian2018} & \\
 & exponential if $p = \Omega(1)$~\cite[Thm.~14]{Qian2018} & \\\cline{1-2}
\multirow{3}{3cm}{\EA,\\bit-wise noise~$(p, 1/n)$} & polynomial if $p = O((\log n)/n^2)$~\cite[Thm.~8]{Qian2018} & & \\
& superpolynomial if $p = \omega((\log n)/n) \cap o(1)$~\cite[Thm.~9]{Qian2018} & & \\
& exponential if $p = \Omega(1)$~\cite[Thm.~10]{Qian2018} & \\\cline{1-2}
\multirow{3}{3cm}{\EA,\\bit-wise noise~$(1,p/n)$} & polynomial if $p = O((\log n)/n^2)$~\cite[Thm.~11]{Qian2018} & \\
& superpolynomial if $p = \omega((\log n)/n) \cap o(1)$~\cite[Thm.~12]{Qian2018} & \\
& exponential if $p = \Omega(1)$~\cite[Thm.~13]{Qian2018} & \\\cline{1-2}
\multirow{2}{3cm}{\EA,\\bit-wise noise~$(p',q/n)$} & \Tstrut{}polynomial if $p := p'\min\{q, 1\} = O((\log n)/n^2)$~\cite[Thm.~5]{Bian2018} & \Tstrut{}\\
 & superpolynomial if $p := p'\min\{q, 1\} = \omega((\log n)/n)$~\cite[Thm.~6]{Bian2018} & \Bstrut{}\\ 
\midrule
\lEA, & $O(\lambda n + n^2)$ if $p \le 0.028/n$ & & $O\big(n^2 \cdot e^{O(pn/\lambda)}\big)$ if $p \le 1/2$ \\
one-bit noise~$p$ & and $72\log n \le \lambda = o(n)$~\cite[Cor.~24]{Giessen2016} & & and $3.42 \log n \le \lambda = O(n)$\\
\bottomrule
\end{tabular}
\end{sidewaystable*}

\citet*{Qian2018} studied the performance of the \EA on \onemax and \LO for a more general prior noise model with parameters $(p, q)$: with probability~$p$ the search point is altered by flipping each bit with probability~$q$. They studied two special cases: $(p, 1/n)$ meaning that with probability~$p$ a standard bit mutation is performed before evaluation and $(1, q)$, which is bit-wise noise with parameter~$q$.
For \LO they improve results from~\cite{Giessen2016}, showing that the \EA runs in polynomial expected time if $p = O((\log n)/n^2)$ and that it runs in superpolynomial time if $p = \omega((\log n)/n)$. This holds for one-bit noise with probability~$p$, the $(p, 1/n)$ model and bit-wise noise with probability~$p/n$ (see Table~\ref{tab:overview}). For bit-wise noise $(1, q)$ with parameter $q = \Omega(1/n)$ the expected time is exponential.

Very recently, \citet{Bian2018} considered the general noise model $(p, q)$ for \onemax and \LO and showed that for \LO the \EA needs polynomial expected time if $p=O((\log n)/n^2)$ or $pq = O((\log n)/n^3)$. It needs superpolynomial time if $p = \omega((\log n)/n)$ and $pq = \omega((\log n)/n^2)$.

In this work we improve previous results for prior noise on the function \LO{}$(x) := \sum_{i=1}^n \prod_{j=1}^i x_j$, counting the number of leading ones in the bit string. This function is of particular interest as it represents a problem where decisions have to be made in sequence in order to reach the optimum, building up the components of a global optimum step by step. In the case of \LO, this is a prefix of ones that is being built up. Problems with similar features are found in combinatorial optimisation, for instances as worst-case examples for finding shortest paths~\cite{Sudholt2011a}.
Multiobjective variants like \textsc{LOTZ} are popular example functions in the theory of evolutionary multiobjective optimisation~\cite{Laumanns2004,Giel2010,NguyenSN15,Qian2013,CovantesOsuna2017a}.

Disruptive mutations can destroy a partial solution, leading to a large fitness loss, such that the algorithm is thrown back and may need a long time to recover. As such, \LO is a prime example of a problem that is very susceptible to noise.

We provide upper and lower bounds on the expected optimisation time of the \EA on \LO, showing that the expected time is in $\Theta(n^2) \cdot \exp(\Theta(\min\{pn^2, n\}))$, which is tight up to constant factors in the exponent of the term $\exp(\Theta(\min\{pn^2, n\}))$ that reflects the slowdown resulting from noise.
This shows that the time is $\Theta(n^2)$ if $p = O(1/n^2)$, polynomial if $p = O((\log n)/n^2)$, superpolynomial if $p = \omega((\log n)/n^2)$ and exponential ($e^{\Theta(n)}$)
if $p = \Omega(1/n)$.
This improves previous negative results that only showed superpolynomial times for $p = \omega((\log n)/n)$, and exponential times for $p = \Omega(1)$, which are both too large by a factor of~$n$.

The upper bound (Section~\ref{sec:general-upper-bound}) is based on a very simple argument: estimating the probability that no noise will occur during a period of time long enough to allow the algorithm to find an optimum without experiencing any noise. A similar argument was used independently in~\cite{Dang-Nhu2018} to derive precise and general results for the \EA on noisy and dynamic \onemax.
The lower bound (Section~\ref{sec:lower-bound}) follows arguments from Rowe and Sudholt~\cite{Rowe2013} who analysed the performance of the non-elitist algorithm (1,$\lambda$)~EA on \LO.

In Section~\ref{sec:offspring-populations} we show an improved upper bound for the \lEA on \LO. Finally, in Section~\ref{sec:noise-helps} we show that on the class of \hurdle problems~\cite{PRUGELBENNETT2004135}, a class of rugged functions with many local optima on an underlying slope, noise helps to overcome local optima, allowing a simple hill climber to succeed that would otherwise fail with overwhelming probability.


This manuscript extends a preliminary version~\cite{Sudholt2018a} that contained parts of the results. In this extension, conditions on bit-wise noise were relaxed in the context of the \EA to allow for larger noise values. An exponential upper bound for the \EA was added to obtain asymptotically tight exponents for all reasonable noise strengths. Several empirical analyses were added to complement the theoretical results
for \LO and \hurdle.

\section{Preliminaries}

Algorithm~\ref{alg:lea} shows the \lEA in the context of prior noise, which includes the \EA as a special case of $\lambda=1$. Here $\noise(x)$ denotes a noisy version of a search point~$x$, according to the given noise model. We assume that all applications of $\noise$ are independent. The \lEA creates $\lambda$ independent offspring, evaluates their noisy fitness, and then picks a best offspring. This offspring is then compared against the parent, whose noisy fitness is evaluated in each generation. This means in particular that an offspring can replace a parent whose real fitness is higher if the parent is misevaluated to a lower noisy fitness, the offspring is misevaluated to a higher noisy fitness, or both.

\begin{algorithm2e}[thb]
  Choose $x$ uniformly at random.\\
  \While{termination criterion not met}{
    \For{$i=1, \dots, \lambda$}{
        Create $y_i$ by copying~$x$ and flipping each bit independently with probability $1/n$.\!\!\!\!\!\!\\
        Evaluate $f_i := f(\noise(y_i))$.
    }
    Choose $z \in P_t$ uniformly at random from $\arg\max\{f_1, \dots, f_\lambda\}$.\\
  \lIf{$f_z \ge f(\noise(x))$}{$x = z$}
}
\caption{\lEA with prior noise}
\label{alg:lea}
\end{algorithm2e}

The optimisation time is defined as the number of fitness evaluations until a global optimum is found for the first time.
We consider the following prior noise models from previous work; asymmetric noise is inspired by an asymmetric mutation operator~\cite{Jansen2010}.

\textbf{One-bit noise$(p)$~\cite{Droste2004,Giessen2016}:} with probability~$1-p$, $\noisy(x) = x$ and otherwise $\noisy(x) = x'$ where in $x'$, compared to~$x$, one bit chosen uniformly at random was flipped.

\textbf{Bit-wise noise$(p, q)$~\cite{Qian2018}:} with probability $1-p$, $\noisy(x) = x$ and otherwise $\noisy(x) = x'$ where in~$x'$, compared to~$x$, each bit was flipped independently with probability~$q$.

\textbf{Asymmetric one-bit noise$(p)$~\cite{Qian2016}:} with probability~$1-p$, $\noisy(x) = x$ and otherwise $\noisy(x) = x'$ where in~$x'$, compared to~$x$, if $x \notin \{0^n, 1^n\}$, with probability $1/2$ a uniform random 0-bit is flipped, with probability $1/2$ a uniform random 1-bit is flipped, and if $x \in \{0^n, 1^n\}$ a uniform random bit is flipped.

The special case $(1, q)$ denotes bit-wise noise as investigated in~\cite{Giessen2016}.
We often write $(p, q/n)$ for bit-wise noise instead of $(p, q)$ as then $q$ plays a similar role to~$p$ in one-bit prior noise~$p$, which allows for a more unified presentation of results: we obtain identical noise thresholds across both models (thresholds for $q$ in the $(1, q)$ model are by a factor of~$n$ smaller than those for~$p$~\cite{Qian2018}).
Note that we do generally allow $q > 1$, while in our preliminary work~\cite{Sudholt2018a} $q$ was restricted to~$q \le 1$.
The conditions from~\cite{Bian2018} for $(p, q/n)$ bit-wise noise simplify to $p \min\{q, 1\} = O((\log n)/n^2)$ for polynomial expected times and $p \min\{q, 1\} = \omega((\log n)/n)$ for superpolynomial times, respectively.

Note that $\prob(\noisy(x) \neq x) = p$ for one-bit noise and asymmetric one-bit noise, and for the bit-wise noise model $(p, q/n)$,
$\prob(\noisy(x) \neq x) = p (1-(1-q/n)^n)$ as noise occurs with probability~$p$ and at least one bit is flipped with probability $1-(1-q/n)^n$.
We simplify the last expression using the following inequalities for all $0 \le p \le 1$ and $\ell \in \mathbb{N}$.
\begin{equation}
\label{eq:theta-min}
\frac{1}{2} \min\{p\ell, 1\} \le \frac{p\ell}{1+p\ell} \le 1-(1-p)^\ell \le \min\{p\ell, 1\}
\end{equation}
The second and third inequality are shown in~\cite[Lemma~6]{Badkobeh2015}, and the first one follows from considering the two cases $p\ell \le 1/2$ and $p\ell > 1/2$.
Thus $\prob(\noisy(x) \neq x)$ is tightly bounded as follows:
\begin{equation}
\label{eq:prob-of-noise-in-bitwise-noise}
\frac{p}{2} \min\{q, \ 1\} \le p (1-(1-q/n)^n) \le p \min\{q, \ 1\}.
\end{equation}
We often limit our considerations to $p \le 1/2$ for one-bit noise as otherwise more than half of the time, the optimum will not be recognised as an optimum. This can lead to counterintuitive effects. For instance, \cite[Theorem~3.3]{Qian2018b} for bit-wise noise with $p=1$ shows that increasing the sample size for the \EA with resampling can turn a polynomial expected time on \LO into an exponential time; this is essentially because states close to the optimum become more appealing than the optimum itself.
For bit-wise noise $(p, q/n)$ we assume $q/n \le 1/2$ as otherwise $\noise(x)$ is more likely return search points that are closer to the bit-wise complement $\overline{x}$ of~$x$ than to~$x$ itself. With $q/n \le 1/2$ the worst possible noise is $q/n=1/2$ where $\noise(x)$ is chosen uniformly at random from the whole search space, irrespective of~$x$.


\section{A Simple and General Upper Bound For Dealing With Uncertainty}
\label{sec:general-upper-bound}

We first present a very simple result that applies in a general setting of optimisation under uncertainty (noise/dynamic changes/etc.). It is formulated for iterative algorithms that maintain a single search point, called \emph{trajectory-based algorithms}, however it is easy to extend the definition to population-based algorithms as well.

Our approach is based on the worst-case median optimisation time, defined as follows.
The definition uses the term \emph{trajectory-based} algorithm to denote an iterative algorithm that maintains one search point in each iteration. The \EA and the \lEA are both trajectory-based algorithms as they both evolve a single search point. The definition also includes randomised local search (RLS), simulated annealing, the (1,$\lambda$)~EA~\cite{Rowe2013} or the Strong Selection Weak Mutation (SSWM) algorithm~\cite{Paixao2016}.
\begin{definition}
For any trajectory-based algorithm $\mathcal{A}$ optimising a fitness function~$f$ let $T_{\mathcal{A}, f}(x)$ be the random first hitting time of a global optimum when starting in~$x$.
We assume hereinafter that each initial search point~$x$ leads to a finite expectation.

We define the worst-case expected optimisation time
$E_{\mathcal{A}, f}$
as
\[
E_{\mathcal{A}, f} := \max_x E(T_{\mathcal{A}, f}(x))
\]
Further define the median optimisation time $M_{\mathcal{A}, f}$
\[
M_{\mathcal{A}, f}(x) := \min \{t \mid \prob(T_{\mathcal{A}, f}(x) \le t) \ge 1/2\}
\]
and the worst-case median optimisation time
\[
M_{\mathcal{A}, f} := \max_x M_{\mathcal{A}, f}(x).
\]
\end{definition}
We omit subscripts if the context is clear.
Applying Markov's inequality for all~$x$, the median worst-case optimisation time is not much larger than the expected worst-case optimisation time
as shown in the following simple theorem\footnote{Much stronger results can be shown, but Theorem~\ref{the:median-vs-expectation} is sufficient for our purposes.}.
\begin{theorem}
\label{the:median-vs-expectation}
For every $\mathcal{A}$ and every~$f$,
$M_{\mathcal{A}, f} \le 2E_{\mathcal{A}, f}$.
\end{theorem}
\begin{proof}
For all~$x$, $M_{\mathcal{A}, f}(x) \le 2E_{\mathcal{A}, f}(x)$ by Markov's inequality.
\end{proof}

The following theorem gives an upper bound on the worst-case expected optimisation time under uncertainty, assuming we do know (an upper bound on) the median worst-case optimisation time in a setting without uncertainty.
\begin{theorem}
\label{the:general-positive-result}
Consider a setting where in each iteration a failure event may occur independently with probability~$0 \le p < 1$.
Consider any function $f$ on which an iterative algorithm $\mathcal{A}$ has worst-case median optimisation time~$M$ if $p=0$.
Then the worst-case expected optimisation time of $\mathcal{A}$ with failure probability~$p$ is at most
\[
2M(1-p)^{-M} \le 2M \cdot e^{pM/(1-p)}.
\]
The statement also holds if $p$ is an upper bound on the probability of a failure and/or $M$ is an upper bound on the described time.
\end{theorem}
\begin{proof}
By definition of the median worst-case optimisation time, if the algorithm experiences $M$ steps without a failure, it will find an optimum with probability at least $1/2$ regardless of the initial search point.
The probability that in a phase of~$M$ steps there will be no failure is at least $(1-p)^M$. Hence the expected waiting time for a phase of~$M$ steps without failures where the algorithm finds an optimum is at most $2M (1-p)^{-M}$ for every initial search point.

The inequality follows from
$
\frac{1}{1-p} = 1 + \frac{p}{1-p} \le e^{p/(1-p)}$.
%
\end{proof}

In the setting of prior noise, Theorem~\ref{the:general-positive-result} implies the following.
\begin{theorem}
\label{the:general-positive-result-prior-noise}
Consider an iterative algorithm $\mathcal{A}$ that evaluates up to~$\nu$ search points in each iteration. For every function $f$ on which $\mathcal{A}$ has worst-case median optimisation time~$M$ without prior noise, its worst-case expected optimisation time is at most
\[
2M(1-p)^{-\nu M} \le 2M \cdot e^{\nu pM/(1-p)}
\]
for each of the following settings:
\begin{enumerate}
\item one-bit prior noise with probability~$p < 1$,
\item bit-wise prior noise $(p', q/n)$ with $q/n \le 1/2$ and $p := p'\min\{q, 1\}$, and
\item asymmetric one-bit prior noise with probability~$p < 1$.
\end{enumerate}
\end{theorem}
\begin{proof}
The probability of noise occurring in one search point is at most~$p$; this is immediate for one-bit noise and it is $p' (1-(1-q/n)^n) \le p'\min\{q, 1\}$ for bit-wise noise by~\eqref{eq:prob-of-noise-in-bitwise-noise}. Since noise is applied to all search points independently, noise occurs in one iteration with probability at most $p^* := 1-(1-p)^\nu$. Invoking Theorem~\ref{the:general-positive-result} with parameter~$p^*$ and the occurrence of noise as failure event yields the first claimed bound. The inequality follows as in the proof of Theorem~\ref{the:general-positive-result}.
\end{proof}

We remark that Theorem~\ref{the:general-positive-result} also applies in many other settings, for example in
\begin{itemize}
\item restart strategies that restart the algorithm in each iteration with probability~$p$,
\item non-elitist algorithms like the (1,$\lambda$)~EA, where the failure event could be defined as the best fitness decreasing,
\item stochastic ageing~\cite{CutelloJCO,Oliveto2014}, an approach from artificial immune systems, where individuals are suddenly killed off with a fixed probability and the failure event is that the whole population happens to die at the same time (which implies a restart),
\item dynamic optimisation where $p$ is the probability of the fitness function changing, if $M$ is taken as (an upper bound for) the worst-case median optimisation time \emph{for all possible fitness functions} that can be attained in the considered dynamic setting.
\end{itemize}

For \LO, Theorem~\ref{the:general-positive-result-prior-noise} implies the following.
\begin{theorem}
\label{the:upper-bound-oneone-LO}
The expected optimisation time of the \EA with prior noise probability~$p \le 1/2$ for each of the settings from Theorem~\ref{the:general-positive-result-prior-noise} on \LO is
\[
O\big(n^2 \cdot e^{O(pn^2)}\big).
\]
This is polynomial if $p=O((\log n)/n^2)$ and $O(n^2)$ if $p = O(1/n^2)$.
\end{theorem}
\begin{proof}
The upper bound follows directly from Theorem~\ref{the:general-positive-result-prior-noise} with $\nu=2$ (as the \EA evaluates parent and offspring in each generation), $2p/(1-p) = O(p)$, and the fact that the worst-case expected optimisation time of the \EA on \LO is $O(n^2)$~\cite{Droste2002}, hence by Theorem~\ref{the:median-vs-expectation} the worst-case median optimisation time is $M = O(n^2)$.
\end{proof}

Despite the simplicity of the above proofs, Theorem~\ref{the:upper-bound-oneone-LO} matches, unifies and generalises the best known results
~\cite{Qian2018,Bian2018} which only classify the expected optimisation time on \LO as being either polynomial, superpolynomial, or exponential (see Table~\ref{tab:overview}). It also gives results for asymmetric one-bit noise, for which no results on \LO are available.

\subsection{An Exponential Upper Bound for Large Noise}

For very large noise levels~$p$, Theorem~\ref{the:upper-bound-oneone-LO} gives an upper bound of essentially $e^{O(pn^2)}$, which can be as bad as $e^{O(n^2)}$ for $p = \Omega(1)$. This is clearly too pessimistic as the expected time to create the optimum by mutation is at most $n^n = e^{n \ln n}$ for \emph{every} fitness function and every initial search point.

We therefore provide a new, tailored upper bound for large noise levels, showing that the expected optimisation time is at most $e^{O(n)}$.
%
To this end, we will prove that the \EA converges to a stationary distribution $\pi$ in which the optimum $1^n$ has stationary mass $\pi(1^n) \ge 2^{-n}$. We then bound the mixing time, that is, the time until the algorithm has approached the stationary distribution such that the optimum is found with a probability close to $\pi(1^n)$. Throughout this section we assume that the reader is familiar with the foundations of Markov chain theory and mixing times as described in relevant text books like~\cite{Levin2008}.

The following lemma shows that transitions to higher fitness values are at least as likely as transitions to lower values.
\begin{lemma}
\label{lem:transitions-towards-higher-LO}
Let $\transition{x}{y}$ denote the probability that the \EA with prior noise transitions from $x$ to~$y$ in one generation. Then for all $x, y$ with $\LO(x) < \LO(y)$ we have
$
\transition{x}{y} \ge \transition{y}{x}
$
in each of the following settings:
\begin{enumerate}
\item one-bit prior noise with probability~$p \le 1/2$,
\item bit-wise prior noise $(p, q)$ with $q \le 1/2$.
\item asymmetric one-bit prior noise with probability~$p \le 1/2$,
\end{enumerate}
\end{lemma}
\begin{proof}
A transition from $x$ to~$y$ is made if and only if mutation of~$x$ results in~$y$ and $y$ is accepted. Since the probability of mutation of~$x$ creating~$y$ is equal to that of mutation of~$y$ creating~$x$, we just need to show that the probability of accepting~$y$ as offspring of~$x$ is no smaller than the probability of accepting~$x$ as offspring of~$y$.

Let $i$ denote the smallest index of any bit flipped in the parent's noise, and $i := \infty$ if there is no such bit. Define $j$ in the same way for the offspring's noise. Abbreviate $\ell := \LO(x)$.

Now, if $i \le j \le \ell$ then the offspring will be accepted regardless of whether the parent is~$x$ or~$y$. If $j < i \le \ell$ the offspring will be rejected in both scenarios. Hence we only need to show the claimed inequality for conditional probabilities assuming $i, j > \ell$.

If $i, j > \ell +1$ then the better search point $y$ will survive, regardless of whether the parent is~$x$ or~$y$.
If $i = \ell +1$ and the parent is~$x$ then the inferior search point $x$ may survive. This case is symmetric to $j=\ell+1$ and $y$ being the parent.
Since $\prob(i=\ell+1)=\prob(j=\ell+1)$ and only one of the previous cases can occur, the probability of $x$ surviving is at most $\prob(i=\ell+1)$.

Thus the claim follows if we can show that
\[
\prob(i, j > \ell+1) \ge \prob(i=\ell+1).
\]
In the symmetric and asymmetric one-bit noise settings, the left-hand side is at least $(1-p)^2 \ge 1/4$ and the right-hand side is at most $p/n \le 1/4$.
For the bit-wise noise setting, if $p \le 1/2$ the left-hand side is at least $1/4$ as above and the right-hand side equals $pq(1-q)^{\ell} \le pq \le 1/4$.
If $p > 1/2$ we argue that the left-hand side is at least $2p(1-p)(1-q)^{\ell+1} \ge p(1-q)^{\ell+1} \ge pq(1-q)^{\ell} = \prob(i=\ell+1)$ as it is sufficient to have noise in exactly one parent, if noise does not flip the first $\ell+1$ bits.
\end{proof}

The exponential upper bound is stated as follows.
\begin{theorem}
\label{the:large-noise-upper-bound}
The expected optimisation time of the \EA with prior noise probability~$p \le 1/2$ for each of the settings from Theorem~\ref{the:general-positive-result-prior-noise}, except for asymmetric one-bit noise, on \LO is at most~$2^{O(n)}$.
\end{theorem}
\begin{proof}
If $p=0$ then the expected optimisation time of the \EA on \LO is $O(n^2) \le 2^{O(n)}$, hence we assume $p > 0$ in the following.

We first show that the \EA on \LO is an ergodic Markov chain, which implies the existence of a stationary distribution~$\pi$. Ergodicity simply follows from the fact that every search point $x$ can be turned into any other search point~$y$ in one generation if mutation of~$x$ creates~$y$ (probability at least $n^{-n}$) and $\LO(\noise(x)) = 0$, which happens with probability at least $p/n > 0$ for one-bit noise, probability at least $p'q > 0$ for bit-wise noise with $p = p'\min\{q, 1\} > 0$ and probability at least $p/(2n) > 0$ for asymmetric one-bit noise.

To prove the claimed inequality $1/\pi(1^n) \le 2^n$ we will use the following property of stationary distributions (cf.\ Proposition~1.19 in~\cite{Levin2008}):
\[
 \pi(x)\cdot \transition{x}{y} = \pi(y)\cdot \transition{y}{x} ,\;\;\;\text{for all}\;\;x,y \in \{0,1\}^n
\]
Since by Lemma~\ref{lem:transitions-towards-higher-LO} $\transition{x}{1^n} \ge \transition{1^n}{x}$ for every search point~$x$, $\pi(1^n) \ge \pi(x)$ for all $2^n$ possible~$x$ and thus $\pi(1^n) \ge 2^{-n}$.

It remains to bound the mixing time, that is, the time until the algorithm has gotten close to the stationary distribution (as will be made precise soon).
Let $p_t$ be the distribution of the current search point at time~$t$.
The difference to the stationary distribution~$\pi$ is described by the \emph{total variation distance} that describes the maximum difference between probabilities for any event $A$:
\[
||p_t - \pi || := \max_{A \subset \Omega}|p_t(A) - \pi(A)|.
\]
In particular, we have $\prob(x_t = 1^n) \ge \pi(1^n) - ||p_t - \pi ||
\ge 2^{-n} - ||p_t - \pi||$.

We now show that $||p_{t} - \pi|| \le 2^{-n-1}$ for a suitable $t = \poly(n) \cdot 2^{O(n)}$. This will be achieved by using a coupling $(X^t, Y^t)$. In a nutshell, a coupling is a pair process where, viewed individually, $X^t$ and $Y^t$ are both faithful copies of the original process, the \EA on \LO. But they may not be independent: they can follow a joint distribution and the coupling ensures that, once they have reached the same state, their states will always be equal. More formally, if $X^t = Y^t$ then $X^{t+1} = Y^{t+1}$. The first point in time where their states become equal, when starting in states $X^0 = x$ and $Y^0 = y$ is called the coupling time $T_{xy}$.

It is known that the tail of the coupling time, or more precisely the tail of the worst-case coupling time for any initial states $x, y$, yields a bound on the total variation distance. Using~\cite[Theorem~5.2]{Levin2008} we get
\[
||p_t - \pi || \le \prob(\max_{x, y} T_{x, y} > t).
\]
We will show the right-hand side becomes less than $2^{-n-1}$ within $2^{O(n)}$ generations\footnote{The author conjectures that this mixing time is, in fact, polynomial, but was unable to prove this. This is left as an open problem for future work.}.

We use the following coupling between two copies $X^t$, $Y^t$ of the \EA, where we identify $X^t$ and $Y^t$ with the \EA{}'s current search points in the respective chains. During mutation, for bits where $X^t$ and $Y^t$ agree we make the same decisions in both Markov chains. Otherwise, with probability $1/n$ we flip the bit in $X^t$ but not in $Y^t$, with probability $1/n$ we flip the bit in $Y^t$ but not in $X^t$, and with the remaining probability $1-2/n$ the bit is not flipped at all.
We further assume that the same noise is applied in both chains. It is easy to verify that both chains, viewed in isolation, represent faithful copies of the \EA on \LO, and that after both chains have reached the same state, their states will always be equal as they experience the same mutations and the same noise.

Let $\equal_t$ denote the size of the largest prefix that is identical in $X^t$ and $Y^t$, i.\,e., $\equal_t = \max\{i \mid X_1^t \dots X_i^t = Y_1^t \dots Y_i^t\}$. Note that if both chains decide to reject their offspring, $\equal_{t+1} = \equal_{t}$ and if both chains decide to accept then $\equal_{t+1} \ge \equal_t$ due to the way mutations are coupled. Once $\equal_t$ has reached a value of~$n$, both chains will always have the same state.

Let $i := \equal_t < n$ then $X_{i+1}^t \neq Y_{i+1}^t$ by definition of $\equal_t$. Assume without loss of generality that $X_{i+1}^t = 0$.
We first show that $\prob(\equal_{t+1} > \equal_t \mid \equal_t, \equal_t < n) \ge 1/(3en)$. A sufficient event is that mutation makes bit~$i+1$ equal in $X^t$ and $Y^t$ and the outcome is accepted in both chains. Mutation flips $X_{i+1}^t$ while not flipping $X_1^t, \dots, X_i^t$ and $Y_1^t, \dots, Y_{i+1}^t$ with probability $1/n \cdot (1-1/n)^i \ge 1/(en)$ as per definition of the coupling mutation flips $X_{i+1}^t$ and does not flip $Y_{i+1}^t$ with probability $1/n$ and every bit $j \le i$ is not flipped in $X^t$ and $Y^t$ with probability $1-1/n$ since $X^t_j = Y^t_j$. The outcome of such a mutation then needs to be accepted despite noise. Let $\alpha_i$ denote the probability of noise flipping any of the first $i$ bits. The offspring will be accepted if noise leaves the first $i$ bits intact, or if noise does flip at least one bit amongst the first $i$ bits in both parent and offspring, but still the offspring's noisy fitness is at least as good as that of its parent. Noting the symmetry in the latter case, the probability of accepting said mutation is at least $(1-\alpha_i)^2 + \alpha_i^2/2 \ge 1/3$ for every possible value~$\alpha_i$. Together, this shows $\prob(\equal_{t+1} > \equal_t \mid \equal_t, \equal_t < n) \ge 1/(3en)$.

Note that the first $i$ bits are identical in the noisy parent evaluation of both $X^t$ and $Y^t$, and they are also identical in the noisy evaluation of both offspring $x', y'$ in $X^t$ and $Y^t$, respectively. If either of these noisy evaluations is less than~$i$, the decision whether to accept or reject is only based on the first $i$ bits and $X^t$ and $Y^t$ make the same decision. The only problematic case is when $\noise(X^t)$, $\noise(Y^t)$, $\noise(x')$, and $\noise(y')$ all have at least~$i$ leading ones as then one Markov chain might accept their offspring while the other might reject theirs. If $\LO(x')$ and $\LO(y')$ are both at least $i$, $\equal_{t+1} \ge \equal_t$ and no harm is done.

However, we might have $\LO(x') < i$ or $\LO(y') < i$ in case mutation destroys the prefix of $i$ leading ones (probability at most $i/n$), but noise flips the same bits, covering up all detrimental mutations. The probability of the latter event is at most $p/n$ for one-bit noise (or 0 in case mutation flipped more than one bit). We call step~$t$ a \emph{relevant step} if $\equal_{t+1} \neq \equal_t$. In a relevant step, the conditional probability of increasing $\equal_t$ is $\Omega(1)$ and the probability of increasing $\equal_t$ in at most $n$ subsequent relevant steps, until $\equal_t = n$ is reached, is at least $(\Omega(1))^n = 2^{-\Omega(n)}$.

In the case of bit-wise noise, the probability of decreasing $\equal_t$ is at most ${q(1-q)^{i-1}}$ as (since $q \le 1/2$) the best case is that mutation has only flipped one bit, which needs to be covered up by noise. The conditional probability of $\equal_t$ increasing in a relevant step is thus at least
\[
\frac{1/(3en)}{q(1-q)^{i-1} + 1/(3en)}
= \frac{1}{1+3enq(1-q)^{i-1}}.
\]
The probability of increasing $\equal_t$ in at most $n$ subsequent relevant steps until a value of~$n$ is reached is thus at least
\begin{align*}
\prod_{i=1}^{n} \frac{1}{1+3enq(1-q)^{i-1}} = \prod_{i=0}^{n-1} \frac{1}{1+3enq(1-q)^{i}}.
\end{align*}
The reciprocal of this expression is upper bounded by
\begin{align*}
\prod_{i=0}^{n-1} (1+3enq(1-q)^{i})
\le\;& \prod_{i=0}^{n-1} \exp(3enq(1-q)^{i})\\
=\;& \exp\left(\sum_{i=0}^{n-1} 3enq(1-q)^{i}\right)\\
\le\;& \exp\left(3enq \sum_{i=0}^{\infty} (1-q)^{i}\right)
= \exp\left(3en\right).
\end{align*}

For both one-bit and bit-wise noise, a relevant step occurs with probability at least $1/(3en)$ (unless the chains have already coupled). Hence the expected waiting time for $n$ relevant steps is at most $3en^2$.
Thus, from any initial configuration of $X^t$ and $Y^t$, the expected time
for a sequence of up to~$n$ relevant steps all increasing $\equal_t$ until the maximum value~$n$ is reached and the chains are coupled is bounded by $\E(\max_{xy} T_{xy}) \le 3en^2 \cdot e^{O(n)} := t^*$. By Markov's inequality, $\prob(\max_{xy} T_{xy} \ge 2t^*) \le 1/2$ and the probability that the process has not coupled within $n+1$ subsequent phases of length $2t^*$ each is at most $2^{-n-1}$.

This shows that the time until the total variation distance to~$\pi$ has decreased to a value of at most $2^{-n-1}$ is $O(n^3) \cdot 2^{O(n)} = 2^{O(n)}$. Then the probability of sampling the optimum in the next generation is at least $\pi(1^n) - 2^{-n-1} \ge 2^{-n-1}$. If the optimum is not found then, we repeat the above arguments. This establishes an upper bound of $O(n^3) \cdot 2^{O(n)} \cdot 2^{n+1} = 2^{O(n)}$.
\end{proof}

\section{A Matching Lower Bound for the \EA on LeadingOnes}
\label{sec:lower-bound}


The arguments from Section~\ref{sec:general-upper-bound} and Theorem~\ref{the:general-positive-result} pessimistically assume that, once noise occurs, the algorithm needs to restart from scratch.
For \LO, and problems with a similar structure, this is not far from the truth. An unlucky mutation can destroy a long prefix of leading ones and the fitness of the current search point can decrease significantly. We will see that then the algorithm comes close to having to start from scratch. Such an effect was already observed and made rigorous in the analysis of island models with migration~\cite{Lassig2013}, separable functions~\cite{Doerr2013}, and for the (1,$\lambda$)~EA on \LO~\cite{Rowe2013}; parts of this section closely follow the proof of Theorem~12 in~\cite{Rowe2013} (but had to be adapted to noisy settings).

The main result of this section is the following.
\begin{theorem}
\label{the:negative-result-1+1}
The expected optimisation time of the \EA with prior noise probability~$p \le 1/2$ for each of the settings from Theorem~\ref{the:general-positive-result-prior-noise} on \LO is
${\Omega\big(n^2 \cdot e^{\Omega(pn^2)}\big)}$ if $p = O(1/n)$ and $e^{\Omega(n)}$ if $p = \omega(1/n)$. This is superpolynomial for $p=\omega((\log n)/n^2)$.
\end{theorem}
Along with Theorems~\ref{the:upper-bound-oneone-LO} and~\ref{the:large-noise-upper-bound} and the fact that polynomial factors only account for a $\pm O(\log n)$ term in the exponent, yielding $e^{\Omega(n)} = \Theta(n^2) \cdot e^{\Omega(n)}$, we get the following result.
\begin{theorem}
\label{the:tight-bounds-1+1}
The expected optimisation time of the \EA
on \LO is
\[
\Theta(n^2) \cdot e^{\Theta(\min\{pn^2, n\})}
\]
for each of the following settings:
\begin{enumerate}
\item one-bit prior noise with probability~$p \le 1/2$ and
\item bit-wise prior noise $(p', q/n)$ with $q/n \le 1/2$ and $p := p'\min\{q, 1\}$.
\end{enumerate}
\end{theorem}
The result is tight up to constants in exponent of the term $\exp(\Theta(\min\{pn^2, n\}))$ that reflects the impact of noise.

Theorem~\ref{the:negative-result-1+1} improves on the best known results, summarised in Table~\ref{tab:overview}.
Note that there is a gap of order $1/n$ between the noise parameter regime $p = \omega((\log n)/n)$ where times are known to be superpolynomial~\cite{Qian2018,Bian2018} and the noise parameter regime~$p = O((\log n)/n^2)$ that led to polynomial upper bounds in~\cite{Qian2018,Bian2018} and in Theorem~\ref{the:upper-bound-oneone-LO}.

Theorem~\ref{the:negative-result-1+1} closes this gap by showing that superpolynomial times already occur for noise parameters $p = \omega((\log n)/n^2)$, which is by a factor of $1/n$ smaller than previous results~\cite{Qian2018,Bian2018}. This shows that the \EA on \LO is highly sensitive to noise, especially since the corresponding threshold for \onemax is at $p = \Theta((\log n)/n)$~\cite{Droste2004,Giessen2016}.
%
Theorem~\ref{the:negative-result-1+1} also unifies and generalises all known results for \LO under prior noise by giving bounds that hold for the whole range of noise parameters~$p$, and for different prior noise models.

In order to prove Theorem~\ref{the:negative-result-1+1}, we first analyse the probability of the fitness dropping significantly.
%
\begin{lemma}
\label{lem:drop-in-fitness}
Consider the setting of Theorem~\ref{the:negative-result-1+1} with a current \LO value of~$i \ge 4$. Then the probability that the \LO value decreases below $i/2$ in one generation is $\Omega(pi^2/n^2)$. This is $\Omega(p)$ if $i = \Omega(n)$.
\end{lemma}
\begin{proof}
Mutation flips a bit at position $\{\lceil n/4 \rceil, \dots, \lfloor n/2 \rfloor\}$ and leaves the other bits unflipped with probability $\Omega(i/n)$. Let $n/4 \le i^* \le i/2$ denote the position of the bit flipped during mutation. Let $i_x$ denote the smallest index of any bit flipped during the parent's noise and $i_x := \infty$ if no such bit exists. Define $i_y$ in the same way for the offspring.
We claim that after a mutation as described above, the probability that the offspring is accepted regardless is $\Omega(pi/n)$. A sufficient condition for this to happen is that $i_x \le n/4 \le i^*$ and $i_y \ge i_x$.

For one-bit noise, we have $\prob(i_x \le i/4) \ge pi/(4n)$.
For asymmetric one-bit noise we get $\prob(i_x \le i/4) \ge pi/(8n)$ as with probability $p/2$, one of at most $n$ 1-bits is flipped.
For bit-wise noise $(p', q/n)$ with $p := p' \min\{q, 1\}$ we have $\prob(i_x \le i/4) \ge p'(1-(1-q/n)^{i/4}) \ge p'/2 \cdot \min\{iq/(4n), 1\}$ by~\eqref{eq:theta-min}. Since $1 \ge i/(4n)$, this is at least $p'/2 \cdot \min\{iq/(4n), i/(4n)\} = p'i/(8n) \cdot \min\{q, 1\} = pi/(8n)$.

For all noise models, we claim that $\prob(i_y \ge i_x \mid i_x \le i^*) \ge 1/2$. If $i_y > i^*$ then $i_y \ge i_x$ with probability~1; otherwise we argue that ${\prob(i_y \ge i_x \mid i_x \le i^*, i_y \le i^*)} \ge {\prob(i_x \ge i_y \mid i_x \le i^*, i_y \le i^*)}$ as parent and offspring are subject to the same independent noise under identical conditions.

If all these events happen, the offspring will appear to be no worse than the parent. Hence the offspring will survive, and its \LO value is at most~$i/2$. Since all events are independent (or conditionally independent), multiplying these probabilities implies the claim.
\end{proof}

As argued in~\cite{Rowe2013} for the (1,$\lambda$)~EA, such a fallback is not too detrimental per se as the \EA might recover from this easily. If the bits between $i/2$ and $i$ have not been flipped during the mutation creating the accepted offspring, the previous leading ones can be easily recovered, in the best case by simply flipping the first 0-bit in the current search point. However, while waiting for such a mutation to happen, all bits between $i/2+1$ and $i$ do not contribute to the fitness. So over time these bits are subjected to random mutations, which are likely to destroy many of the former leading ones. In other words, after a fallback previous leading ones are forgotten quickly.

The last fact was formalised in~\cite[Lemma~3]{Lassig2013} stated below. The lemma states that the probability distribution of a bit subjected to random mutations rapidly approaches a uniform distribution.
\begin{lemma}[Adapted from L{\"a}ssig and Sudholt~\cite{Lassig2013}]
\label{lem:mixing}
Let $x^0, x^1, \dots, x^t$ be a sequence of random bit values such that $x^{j+1}$ results from $x^j$ by flipping the bit $x^j$ independently with probability $1/n$.
Then for every $t \in \mathbb{N}$
\[
\Prob(x^t = 1) \le \frac{1}{2} \left(1 + \left(1 - \frac{2}{n}\right)^t\right)\;.
\]
\end{lemma}

We now say that the \EA \emph{falls back} if, starting from a fitness at least $f^* := 2n/3$, the algorithm drops to a fitness of $i^*$ for some $i^* \le n/2$.
We speak of a \emph{lasting fallback} if in the $2n/(1-p)$ generations directly following a fallback the following holds:
\begin{enumerate}
\item all acceptance decisions are made independently from bit values at positions ${i^* + 2, \dots, n}$,
\item bit $i^*+1$ is never flipped during mutation and
\item in at least $n/2$ generations the offspring is accepted.
\end{enumerate}
A lasting fallback implies that the fitness remains at most $i^*$ during at least $n/2$ accepted steps. In these accepted steps, the bits at positions $i^*+2, \dots, n$ are mutated independently from acceptance decisions and hence take on a near-random state.

We remark that in a noise-free setting, so long as bit $i^*+1$ is never flipped, the acceptance decisions would trivially be independent from bit positions $i^*+2, \dots, n$. In a setting with noise, however, these bits might play a role as bit $i^*+1$ might be flipped by noise, and then the acceptance decision might depend on further bits. Hence more careful arguments are needed.

We also say that the initial search point is a lasting fallback if its fitness is at most $n/2$. If $i^*$ is the initial fitness, the bits at positions $i^*+2, \dots, n$ take on a uniform random state.

The following lemma estimates probabilities for fallbacks and lasting fallbacks.
\begin{lemma}
\label{lem:fallback}
Consider any of the settings described in Theorem~\ref{the:general-positive-result-prior-noise}. If $p \le 1/2$ and the current fitness is at least~$f^*$, the probability of one generation yielding a fallback is $\Omega(p)$.
Additionally, the probability of a fallback becoming a lasting fallback is~$\Omega(1)$.
\end{lemma}
\begin{proof}
The first statement follows from Lemma~\ref{lem:drop-in-fitness} as halving the current fitness results in a search point of fitness at most~$n/2$.

It remains to estimate the probability of a fallback becoming a lasting fallback. Let $i^*$ be the fitness obtained during a fallback and let $x_t$ be the parent in generation~$t$. Abbreviate $i_t := \LO(x_t)$. We call a generation~$t$ \emph{good} if
\begin{itemize}
\item bit $i_t+1$ is not flipped during mutation and
\item bit $i_t+1$ is set to~$0$ in the noisy parent.
\end{itemize}
In a good generation the noisy fitness of the parent is at most $i_t$, hence the offspring is accepted if and only if its noisy fitness is at least $i_t$. This decision only depends on bits at positions $1, \dots, i_t$ and is independent from bits at positions $i_t+2, \dots, n$.

Moreover, in a good generation~$t$ we have $i_{t+1} \le i_t$ as the fitness cannot increase if bit $i_t+1$ is not flipped during mutation.
If all generations since the fallback have been good then $i_t \le i^*$ and decisions are independent from bits $i^*+2, \dots, n$ as claimed.

We estimate the probability of all $2n/(1-p)$ generations being good. For any generation~$t$, the probability of the first event is $1-1/n$. The probability of the second event is at least $1 - 1/n - p/n \ge 1-2/n$ as bit $i_t+1$ can only be set to~1 if it is mutated or flipped during noise. The probability of noise flipping any fixed bit is at most $p/n$ in all considered noise settings.
Hence the probability of a generation~$t$ being good is at least $1-3/n$ by a union bound and the probability that all $2n/(1-p)$ generations are good is $(1-3/n)^{2n/(1-p)} = \Omega(1)$.

Assuming that these generations are all good, we finally estimate the number of accepted generations under this condition. Using $\Prob(A \mid B) = \Prob(A \cap B)/P(B) \ge \Prob(A \cap B)$, we lower-bound the probability of a generation~$t$ being accepted and good. This happens if bits $1, \dots, i_t+1$ are not flipped during mutation (probability at least $(1-1/n)^n$), bit $i_t+1$ is set to~0 in the noisy parent (probability at least $1-2/n$ as estimated above) and the offspring does not suffer from noise (probability at least $1-p$). Together, the probability of an accepted generation conditional on it being good is at least $(1-1/n)^n \cdot (1-2/n) \cdot (1-p) \ge (1-p)/3$ if $n$ is large enough.
The expected number of accepted generations in $2n/(1-p)$ good generations is at least $2n/3$ and by Chernoff bounds, the probability of having at least $n/2$ accepted generations is $1-2^{-\Omega(n)}$.

Together, all three criteria in the definition of lasting fallbacks hold with probability $\Omega(1)$.
\end{proof}


After a lasting fallback has occurred, the \EA with overwhelming probability needs some time in order to recover. Specifically, at least $cn^2$ generations are needed to increase the best fitness since the latest lasting fallback by at least~$n/6$.
\begin{lemma}
\label{lem:recovery-after-fallback}
Let $t$ be the latest generation where a fallback became a lasting fallback or $t = 0$ if no lasting fallback occurred.
Let $B_t$ be the best fitness found since generation~$t$.
With probability $1-e^{-\Omega(n)}$, for a small constant~$c > 0$, $B_{t + cn^2} < B_{t} + n/6$.
\end{lemma}
\begin{proof}
We pessimistically overestimate the probability of a fitness improvement due to the effects of noise in generations from $t$ to~$t+cn^2$: we assume that noise never leads to a decrease in the number of leading ones. Secondly, we call a step \emph{successful} if the first 0-bit is flipped during mutation or if it is flipped during the parent's or offspring's noise. In this case we assume that this bit becomes part of the leading ones for the next generation and the next parent's fitness is determined by the position of the first 0-bit amongst the following bits. The probability of a successful step is still bounded from above by $3/n$.

A lasting fallback implies that at any generation from~$t$, all bits at positions $\{B_t + 1, \dots, n\}$ have been subjected to mutation at least $t_{\mathrm{mix}} = n/2$ times and these mutations were independent of the acceptance decision (by definition of a lasting fallback). Every mutation flips each of these bits independently with probability~$1/n$, leaving the bits in a random state. We apply the \emph{principle of deferred decisions}~\cite[page~9]{Mitzenmacher2005} and determine the current bit value for these bits at the time these bits first have a chance to become part of the leading ones in an offspring. By Lemma~\ref{lem:mixing} we know that then the probability such a bit is set to~1 is at most
\[
\frac{1}{2} \left(1 + \left(1 - \frac{2}{n}\right)^{n/2}\right) \le \frac{1}{2} \left(1 + \frac{1}{e}\right) = \frac{e+1}{2e}.
\]
Note that due to our pessimistic assumptions concerning successful steps, the bits following the first 0-bit will always be irrelevant for the decision whether or not to accept the offspring. Hence the above probability bound also holds after generation~$t$.

A necessary condition for increasing the best fitness by at least $n/6$ in $cn^2$ generations, $c$ a positive constant chosen later, is that either
\begin{enumerate}
\item among $cn^2$ mutations at least $6cn$ steps are successful or
\item during at most $6cn$ successful steps the total fitness gain is at least $n/6$.
\end{enumerate}
The probability of a successful step is always at most~$3/n$ as mentioned earlier. By standard Chernoff bounds, the probability for the first event is at most $e^{-\Omega(n)}$. The total fitness gain is given by the number of improvements---at most $6cn$---plus a sum of up to~$6cn$ geometric random variables to account for additional bits gained (these additional bits are often called ``free riders''). By Theorem~5 in~\cite{Baswana2009}, we get that the probability of a fitness gain of $n/6$ is $e^{-\Omega(n)}$, provided that $c$ is small enough.
\end{proof}

\begin{lemma}
\label{lem:fallback-occurs}
Let $c > 0$ be any constant. Within $cn^2$ generations where the current fitness is larger than~$f^*$, a lasting fallback occurs with probability at least $1 - e^{-\Omega(pn^2)}$.
\end{lemma}
\begin{proof}
The probability of a fallback occurring is $\Omega(p)$, and then it becomes lasting with probability $\Omega(1)$. Note that the time until a fallback potentially becomes a lasting fallback (whether it does or not) is not counted towards the $cn^2$ generations from the statement as during this time the fitness is smaller than~$f^*$.

So the probability that no lasting fallback occurs is at most
\[
\left(1 - \Omega(p)\right)^{cn^2} \le e^{-\Omega(pn^2)}. \qedhere
\]
\end{proof}

Now we prove Theorem~\ref{the:negative-result-1+1}.
\begin{proofof}{Theorem~\ref{the:negative-result-1+1}}
With probability $1 - 2^{-\Omega(n)}$ the initial search point has fitness less than~$n/2$, so the \EA starts with a lasting fallback.
As the fitness after initialisation and after every lasting fallback is at most $n/2$, by Lemma~\ref{lem:recovery-after-fallback}, reaching a fitness of at least $f^*$ from there takes time at least $cn^2$ with overwhelming probability, for a suitably small constant~$c > 0$.
Applying Lemma~\ref{lem:recovery-after-fallback} every time the fitness increases to at least~$f^*$, the \EA does not find an optimum within the next $cn^2$ generations where the fitness is at least~$f^*$, with overwhelming probability.
But by Lemma~\ref{lem:fallback-occurs} during these $cn^2$ generations another lasting fallback occurs, with overwhelming probability.
We iterate this argument until a failure occurs. The largest failure probability is $e^{-\Omega(pn^2)}$ if $p = O(1/n)$, hence in expectation we can iterate this argument at least $e^{\Omega(pn^2)}$ times, each iteration taking time at least $cn^2$ (from the time it takes to reach fitness $f^*$ after a lasting fallback). If $p = \omega(1/n)$, the largest failure probability is $e^{-\Omega(n)}$ and in expectation we can iterate this argument for $e^{\Omega(n)}$ generations. Together, this proves the claim.
\end{proofof}

\section{Improved Results for Offspring Populations}
\label{sec:offspring-populations}

The general Theorem~\ref{the:general-positive-result} can also be used in the context of offspring populations in the \lEA, in order to quantify the robustness of evolutionary algorithms with offspring populations to noise.
Offspring populations can reduce the probability of the current fitness decreasing. The current fitness can decrease in two different ways:
\begin{enumerate}
\item the current search point may be misevaluated as having a poor fitness, and then be replaced by an offspring that is worse than the parent in real fitness or
\item the current search point may be replaced by an offspring where mutation has led to poor real fitness, but noise happens to misevaluate the offspring as having a high fitness, thus replacing its parent. Here noise essentially needs to make the same bit-flips as the preceding mutation to cover up the effect of mutation.
\end{enumerate}

The first failure can be avoided if there is a clone of the current search point where no prior noise has occurred. A large offspring population can amplify this probability.
\begin{lemma}
\label{lem:all-clones-have-noise}
Consider the \lEA in a prior noise model where
$\prob(\noise(y) \neq y) \le p$ for all search points~$y$.
Then for all current search points $x$ the probability that all copies of $x$ among parent and offspring are affected by noise is at most
\[
p\left(1-\left(1-\frac{1}{n}\right)^n(1-p)\right)^\lambda
= p\left(\frac{e-(1-p)}{e}\right)^{\lambda} \cdot \exp(O(\lambda/n)).
\]
\end{lemma}
\begin{proof} 
Let $q := (1-1/n)^n$ abbreviate the probability of creating a clone of the parent for an offspring. The probability of creating exactly $i$ clones is ${\binom{\lambda}{i} q^i (1-q)^{\lambda-i}}$, and then the probability that all $i+1$ copies of~$x$ (including the parent) are affected by noise is at most~$p^{i+1}$. Hence the sought probability is
\begin{align*}
\sum_{i=0}^\lambda \binom{\lambda}{i} q^i (1-q)^{\lambda-i} p^{i+1}
=\;& p \sum_{i=0}^\lambda \binom{\lambda}{i} (pq)^i (1-q)^{\lambda-i}\\
=\;& p (1-q+pq)^\lambda\\
=\;& p (1-q(1-p))^\lambda
\end{align*}
where we have used the binomial theorem in the penultimate equality. Plugging in $(1-1/n)^n$ for $q$ yields the claimed result.
%
For the second bound we use $(1-1/n)^n = (1-1/n)(1-1/n)^{n-1} \ge (1-1/n) \cdot 1/e$,
\begin{align*}
& \left(1-\left(1 - \frac{1}{n}\right)^n (1-p)\right)^\lambda\\
\le\;& \left(1-\frac{1}{e} \left(1 - \frac{1}{n}\right) (1-p)\right)^\lambda\\
=\;&\left(1 - \frac{1}{e}(1-p)\right)^{\lambda} \left(\frac{1-\frac{1}{e}\left(1 - \frac{1}{n}\right) (1-p)}{1-\frac{1}{e}(1-p)}\right)^\lambda\\
=\;&\left(1 - \frac{1}{e}(1-p)\right)^{\lambda} \left(1 + \frac{\frac{1}{en} (1-p)}{1-\frac{1}{e}(1-p)}\right)^\lambda\\
=\;&\left(\frac{e-(1-p)}{e}\right)^{\lambda} \left(1 + \frac{\frac{1}{n}(1-p)}{e-(1-p)}\right)^\lambda\\
\le\;&\left(\frac{e-(1-p)}{e}\right)^{\lambda} \exp\left(\frac{\lambda}{n} \cdot \frac{1-p}{e-(1-p)}\right). \qedhere
\end{align*}
\end{proof}

Our aim is to apply Theorem~\ref{the:general-positive-result} where the failure event is the union of the event described in Lemma~\ref{lem:all-clones-have-noise} and other events described later. However, we still need a bound on the worst-case median optimisation time, or (by Theorem~\ref{the:median-vs-expectation}) the worst-case expected optimisation time, assuming that the algorithm always retains at least one copy of the current search point.

Note that we cannot simply use a runtime result for the \lEA without noise as noise can still affect the generated offspring; the only condition we can rely on is that we cannot lose all copies of the current search point. If noise is disruptive, the \lEA may behave like having a smaller effective offspring population, the size of which is random. Note that we cannot pessimistically use a bound on the \EA to upper bound the time of the \lEA in this setting as different offspring population sizes can affect search dynamics in unforeseen ways. \citet{Jansen2005a} presented a problem class where different offspring population sizes lead to very different performance.

The following theorem gives improved upper bounds for one-bit noise and bit-wise noise\footnote{We exclude asymmetric bit-wise noise as the probability of flipping a 1-bit may be $\omega(1/n)$ in case there are $o(n)$ leading ones, and only $o(n)$ 1-bits in total. We cannot exclude that this happens, though it seems highly unlikely in the light of Lemma~\ref{lem:mixing}. We also restrict bit-wise noise to $q/n \le 1/n$.}.
\begin{theorem}
\label{the:lambda-on-LO}
The expected number of function evaluations for the \lEA with prior noise parameter~$p \le 1/2$ on \LO with $\log_{\frac{e}{e-1/2}}(n) \le \lambda = O(n)$ is
\[
O\left(n^2 \cdot e^{O(pn/\lambda)}\right)
\]
in each of the following settings:
\begin{enumerate}
\item one-bit prior noise with probability~$p < 1$ and
\item bit-wise prior noise $(p', q/n)$ with $q/n \le 1/n$ and $p := p'\min\{q, 1\}$.
\end{enumerate}
This is polynomial if $p = O((\lambda \log n)/n)$ and $O(n^2)$ if $p = O(\lambda/n)$.
\end{theorem}
The exponent is smaller compared to the upper bound for the \EA by a factor of order $\lambda n$, and thus the threshold for~$p$ for which polynomial times are guaranteed increases by the same factor. The threshold between polynomial and superpolynomial times could be higher as we do not have a corresponding lower bound.

Theorem~\ref{the:lambda-on-LO} improves and generalises the best known result for the \lEA~\cite[Corollary~24]{Giessen2016} which requires $p = O(1/n)$ and $\lambda \ge 72 \log n$ and gives a time bound of $O(\lambda n + n^2)$. This is $O(n^2)$ as the authors also assume $\lambda = o(n)$. Our result covers the whole parameter range for~$p$ up to~$1/2$ and also identifies a functional relationship between $p$ and~$\lambda$ that guarantees robustness to noise.

\begin{proofof}{Theorem~\ref{the:lambda-on-LO}}
We estimate the probability of the following failure events in order to apply a union bound later on.

\paragraph{Failure event $E_1$:} all copies of the current search point are affected by noise.
By Lemma~\ref{lem:all-clones-have-noise}, this probability is at most
\[
p_1 := \mathord{O}\mathord{\left(p\left(\frac{e-(1-p)}{e}\right)^{\lambda}\right)}
\le \mathord{O}\mathord{\left(p\left(\frac{e-1/2}{e}\right)^{\lambda}\right)}
= \mathord{O}\mathord{\left(\frac{p}{n}\right)}.
\]

\paragraph{Failure event $E_2$:} the best offspring is evaluated as having the parent's fitness, and the offspring $y$ chosen to replace the parent carries disruptive mutations that were undone by noise, i.\,e.\ $\LO(y) < \LO(\noise(y)) = \LO(x)$. The probability for this to happen is at most
\[
p_2 := \frac{p}{n}
\]
as noise has to flip at least one specific bit.

\paragraph{Failure event $E_3$:} there is an offspring $y$ that carries disruptive mutations, but is being evaluated as being better than the parent, i.\,e.\ $\LO(y) < \LO(x)$ and $\LO(\noise(y)) > \LO(x)$.
For each offspring where mutation flips one of the leading ones, two events may occur: if mutation flips the first 0-bit, noise in an offspring has to undo all mutations of the leading ones. This has probability at most $p/n^2$. Otherwise, noise has to undo all mutations of the leading ones and flip the first 0-bit at the same time. This is impossible under one-bit noise, and has probability at most $p/n^2$ under bit-wise noise. Along with a union bound over these two events and $\lambda$ offspring,
\[
p_3 \le \frac{2p\lambda}{n^2} = \mathord{O}\mathord{\left(\frac{p}{n}\right)}.
\]
As long as no failure occurs, the current fitness of the \lEA cannot decrease. We now show that, conditional on no failure occurring, the expected worst-case number of generations of the \lEA is bounded by $O(n + n^2/\lambda) = O(n^2/\lambda)$.

The probability of one offspring increasing the current fitness is at least ${(1-p)/(en)}$ as it suffices to flip the first 0-bit and not to flip any of the other bits, and to have the offspring being evaluated correctly. The probability that this happens in at least one of the $\lambda$ offspring and the parent is evaluated correctly is at least
\[
(1-p)\left(1 - \left(1 - \frac{1-p}{en}\right)^\lambda\right) \ge \frac{(1-p)^2\lambda/(en)}{1+(1-p)\lambda/(en)}
= \Omega\left(\frac{\lambda}{n}\right)
\]
where the inequality follows from~\cite[Lemma~6]{Badkobeh2015}.
The expected time to increase the best fitness is thus $O(n/\lambda)$, and since the fitness only has to be increased at most $n$ times, an upper bound of $O(n^2/\lambda)$ generations follows, for every initial search point. The same bound also holds for the worst-case median optimisation time by Theorem~\ref{the:median-vs-expectation}.

Now the result follows from applying Theorem~\ref{the:general-positive-result} with a time bound of $O(n^2/\lambda)$ and a failure probability bound of $p_1 + p_2 + p_3 = O(p/n)$, and multiplying the number of generations by $\lambda$ for the number of function evaluations.
\end{proofof}

\subsection{Experiments for LeadingOnes}

We also performed experiments to see the threshold behaviour more clearly and to get further insights into the search dynamics in the presence of noise.

Figure~\ref{fig:averages-1+lambda-on-LO} shows the average optimisation times over 1000 runs of the \lEA on 100-bit \LO with $\lambda \in \{1, 2, 4, 8, 16\}$ for both
one-bit prior noise with probability~$p$ and bit-wise prior noise $(1, q/n)$. For both noise models the parameter was varied exponentially: $p
 \in \{2^{-20}, 2^{-19}, \dots, 2^{-1}\}$ and $q \in \{2^{-20}, 2^{-19}, \dots, 2^{0}\}$.
Runs were stopped after $10n^2 = 10^5$ generations or when the optimum was found. For the \EA with one-bit noise we can see that for small noise values like $p \in \{2^{-20}, \dots, 2^{-15}\}$ the averages seem unaffected by the noise parameter, as noise occurs too rarely to have a noticeable effect. When increasing~$p$, the average time increases slightly before shooting up around $p =2^{-8}$ and hitting the generation limit at $p=2^{-6}$ in nearly all runs. This clearly shows that and how the expected optimisation time grows exponentially in $pn^2$ in this regime.

\begin{figure*}[tb]
\centering
\subfigure[one-bit noise]{
\begin{tikzpicture}[scale=0.76]
\begin{axis}[legend cell align=left, legend pos=north west, xlabel={$\log(p)$}, 
ymin=0, ymax=100000, xmin=-20, xmax=0,
forget plot style={opacity=0.35,mark=none},
]
\pgfplotstableread[header=false]{experiments/runtime-(1+1)EA-n=100-onebit-p-increasing-to-0.5-stopped-after-10n2-1000runs.txt}\pOneOne
\pgfplotstableread[header=false]{experiments/runtime-(1+2)EA-n=100-onebit-p-increasing-to-0.5-stopped-after-10n2-1000runs.txt}\pOneTwo
\pgfplotstableread[header=false]{experiments/runtime-(1+4)EA-n=100-onebit-p-increasing-to-0.5-stopped-after-10n2-1000runs.txt}\pOneFour
\pgfplotstableread[header=false]{experiments/runtime-(1+8)EA-n=100-onebit-p-increasing-to-0.5-stopped-after-10n2-1000runs.txt}\pOneEight
\pgfplotstableread[header=false]{experiments/runtime-(1+16)EA-n=100-onebit-p-increasing-to-0.5-stopped-after-10n2-1000runs.txt}\pOneSixteen
\addplot table[x expr=(-21+\thisrow{1}), y index=7]{\pOneOne};
\addplot[plotcolor1,forget plot] table[x expr=(-21+\thisrow{1}), y expr=\thisrow{7} + \thisrow{11}]{\pOneOne};
\addplot[plotcolor1,forget plot] table[x expr=(-21+\thisrow{1}), y expr=\thisrow{7} - \thisrow{11}]{\pOneOne};
\addplot table[x expr=(-21+\thisrow{1}), y index=7]{\pOneTwo};
\addplot[plotcolor2,forget plot] table[x expr=(-21+\thisrow{1}), y expr=\thisrow{7} + \thisrow{11}]{\pOneTwo};
\addplot[plotcolor2,forget plot] table[x expr=(-21+\thisrow{1}), y expr=\thisrow{7} - \thisrow{11}]{\pOneTwo};
\addplot table[x expr=(-21+\thisrow{1}), y index=7]{\pOneFour};
\addplot[plotcolor3,forget plot] table[x expr=(-21+\thisrow{1}), y expr=\thisrow{7} + \thisrow{11}]{\pOneFour};
\addplot[plotcolor3,forget plot] table[x expr=(-21+\thisrow{1}), y expr=\thisrow{7} - \thisrow{11}]{\pOneFour};
\addplot table[x expr=(-21+\thisrow{1}), y index=7]{\pOneEight};
\addplot[plotcolor4,forget plot] table[x expr=(-21+\thisrow{1}), y expr=\thisrow{7} + \thisrow{11}]{\pOneEight};
\addplot[plotcolor4,forget plot] table[x expr=(-21+\thisrow{1}), y expr=\thisrow{7} - \thisrow{11}]{\pOneEight};
\addplot table[x expr=(-21+\thisrow{1}), y index=7]{\pOneSixteen};
\addplot[plotcolor5,forget plot] table[x expr=(-21+\thisrow{1}), y expr=\thisrow{7} + \thisrow{11}]{\pOneSixteen};
\addplot[plotcolor5,forget plot] table[x expr=(-21+\thisrow{1}), y expr=\thisrow{7} - \thisrow{11}]{\pOneSixteen};
\legend{\EA, (1+2)~EA, (1+4)~EA, (1+8)~EA, (1+16)~EA}
\end{axis}
\end{tikzpicture}
}
\subfigure[bit-wise noise]{
\begin{tikzpicture}[scale=0.76]
\begin{axis}[legend cell align=left, legend pos=north west, xlabel={$\log(q)$}, 
ymin=0, ymax=100000, xmin=-20, xmax=0,
forget plot style={opacity=0.35,mark=none},
]
\pgfplotstableread[header=false]{experiments/runtime-rerun-(1+1)EA-n=100-bitwise-p-increasing-to-0.5-stopped-after-10n2-1000runs.txt}\pOneOne
\pgfplotstableread[header=false]{experiments/runtime-(1+2)EA-n=100-bitwise-p-increasing-to-0.5-stopped-after-10n2-1000runs.txt}\pOneTwo
\pgfplotstableread[header=false]{experiments/runtime-(1+4)EA-n=100-bitwise-p-increasing-to-0.5-stopped-after-10n2-1000runs.txt}\pOneFour
\pgfplotstableread[header=false]{experiments/runtime-(1+8)EA-n=100-bitwise-p-increasing-to-0.5-stopped-after-10n2-1000runs.txt}\pOneEight
\pgfplotstableread[header=false]{experiments/runtime-(1+16)EA-n=100-bitwise-p-increasing-to-0.5-stopped-after-10n2-1000runs.txt}\pOneSixteen
\addplot table[x expr=(-21+\thisrow{1}), y index=7]{\pOneOne};
\addplot[plotcolor1,forget plot] table[x expr=(-21+\thisrow{1}), y expr=\thisrow{7} + \thisrow{11}]{\pOneOne};
\addplot[plotcolor1,forget plot] table[x expr=(-21+\thisrow{1}), y expr=\thisrow{7} - \thisrow{11}]{\pOneOne};
\addplot table[x expr=(-21+\thisrow{1}), y index=7]{\pOneTwo};
\addplot[plotcolor2,forget plot] table[x expr=(-21+\thisrow{1}), y expr=\thisrow{7} + \thisrow{11}]{\pOneTwo};
\addplot[plotcolor2,forget plot] table[x expr=(-21+\thisrow{1}), y expr=\thisrow{7} - \thisrow{11}]{\pOneTwo};
\addplot table[x expr=(-21+\thisrow{1}), y index=7]{\pOneFour};
\addplot[plotcolor3,forget plot] table[x expr=(-21+\thisrow{1}), y expr=\thisrow{7} + \thisrow{11}]{\pOneFour};
\addplot[plotcolor3,forget plot] table[x expr=(-21+\thisrow{1}), y expr=\thisrow{7} - \thisrow{11}]{\pOneFour};
\addplot table[x expr=(-21+\thisrow{1}), y index=7]{\pOneEight};
\addplot[plotcolor4,forget plot] table[x expr=(-21+\thisrow{1}), y expr=\thisrow{7} + \thisrow{11}]{\pOneEight};
\addplot[plotcolor4,forget plot] table[x expr=(-21+\thisrow{1}), y expr=\thisrow{7} - \thisrow{11}]{\pOneEight};
\addplot table[x expr=(-21+\thisrow{1}), y index=7]{\pOneSixteen};
\addplot[plotcolor5,forget plot] table[x expr=(-21+\thisrow{1}), y expr=\thisrow{7} + \thisrow{11}]{\pOneSixteen};
\addplot[plotcolor5,forget plot] table[x expr=(-21+\thisrow{1}), y expr=\thisrow{7} - \thisrow{11}]{\pOneSixteen};
\legend{\EA, (1+2)~EA, (1+4)~EA, (1+8)~EA, (1+16)~EA}
\end{axis}
\end{tikzpicture}
}
\caption{Average number of generations over 1000 runs for the \lEA with $\lambda \in \{1, 2, 4, 8, 16\}$ on \LO ($n=100$) with one-bit prior noise with probability~$p \in \{2^{-20}, 2^{-19}, \dots, 2^{-1}\}$ and bit-wise prior noise $(1, q/n)$ with $q \in \{2^{-20}, 2^{-19}, \dots, 2^{0}\}$.
Runs were stopped after $10n^2$ generations. Transparent lines show means $\pm$ standard deviation.}
\label{fig:averages-1+lambda-on-LO}
\end{figure*}
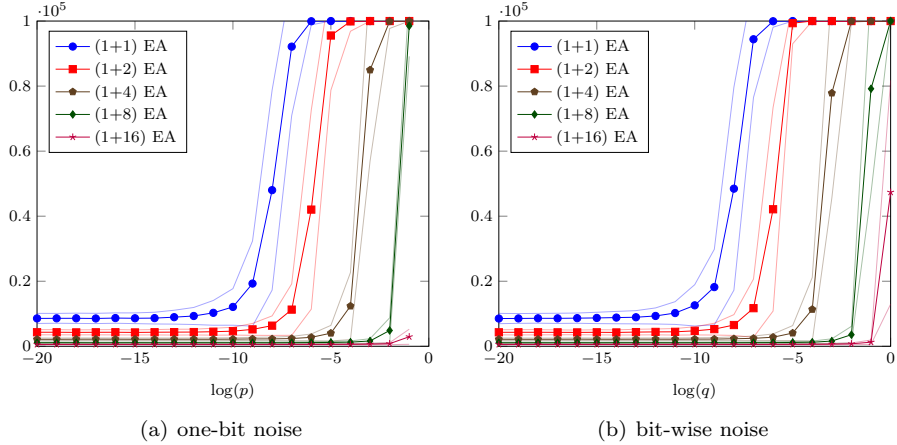

Figure~\ref{fig:averages-1+lambda-on-LO} further shows how offspring populations can shift the threshold between efficient and inefficient times towards higher values of~$p$. Even very small offspring population sizes~$\lambda$ have a significant effect. For instance, the (1+8)~EA is still efficient for $p=1/4$ and only becomes inefficient for $p=1/2$. The (1+16)~EA is efficient even for $p=1/2$. Note that the curves for all \lEA{}s have a very similar shape, independent of~$\lambda$; they just appear to be shifted towards different values of~$p$. This matches our theoretical results as the exponential term $e^{O(pn/\lambda)}$ contains the ratio $p/\lambda$, indicating that the noise strength can be compensated by the offspring population size in a linear fashion.

 Comparing plots for one-bit noise and bit-wise noise, the curves look almost identical.

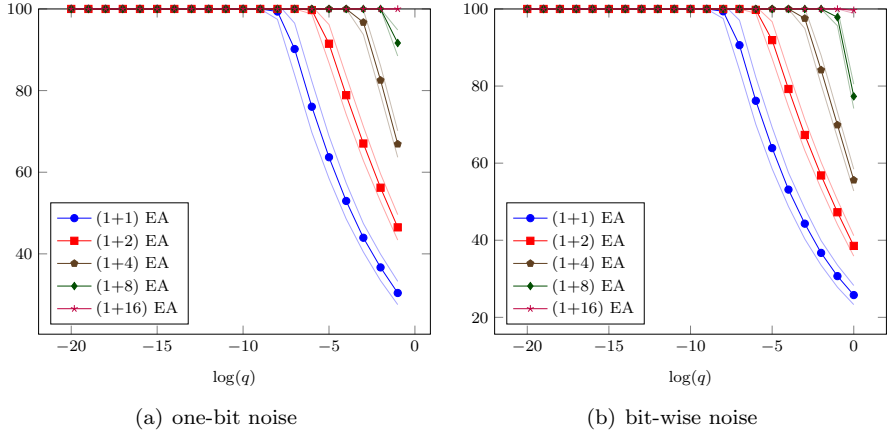
\begin{figure*}[bt]
\centering
\subfigure[one-bit noise]{
\begin{tikzpicture}[scale=0.76]
\begin{axis}[legend cell align=left, legend pos=south west, xlabel={$\log(q)$}, 
ymax=100,
forget plot style={opacity=0.35,mark=none},
]
\pgfplotstableread[header=false]{experiments/bestfitness2-(1+1)EA-n=100-onebit-p-increasing-to-0.5-stopped-after-10n2-1000runs.txt}\pOneOne
\pgfplotstableread[header=false]{experiments/bestfitness2-(1+2)EA-n=100-onebit-p-increasing-to-0.5-stopped-after-10n2-1000runs.txt}\pOneTwo
\pgfplotstableread[header=false]{experiments/bestfitness2-(1+4)EA-n=100-onebit-p-increasing-to-0.5-stopped-after-10n2-1000runs.txt}\pOneFour
\pgfplotstableread[header=false]{experiments/bestfitness2-(1+8)EA-n=100-onebit-p-increasing-to-0.5-stopped-after-10n2-1000runs.txt}\pOneEight
\pgfplotstableread[header=false]{experiments/bestfitness2-(1+16)EA-n=100-onebit-p-increasing-to-0.5-stopped-after-10n2-1000runs.txt}\pOneSixteen
\addplot table[x expr=(-21+\thisrow{1}), y index=7]{\pOneOne};
\addplot[plotcolor1,forget plot] table[x expr=(-21+\thisrow{1}), y expr=\thisrow{7} + \thisrow{11}]{\pOneOne};
\addplot[plotcolor1,forget plot] table[x expr=(-21+\thisrow{1}), y expr=\thisrow{7} - \thisrow{11}]{\pOneOne};
\addplot table[x expr=(-21+\thisrow{1}), y index=7]{\pOneTwo};
\addplot[plotcolor2,forget plot] table[x expr=(-21+\thisrow{1}), y expr=\thisrow{7} + \thisrow{11}]{\pOneTwo};
\addplot[plotcolor2,forget plot] table[x expr=(-21+\thisrow{1}), y expr=\thisrow{7} - \thisrow{11}]{\pOneTwo};
\addplot table[x expr=(-21+\thisrow{1}), y index=7]{\pOneFour};
\addplot[plotcolor3,forget plot] table[x expr=(-21+\thisrow{1}), y expr=\thisrow{7} + \thisrow{11}]{\pOneFour};
\addplot[plotcolor3,forget plot] table[x expr=(-21+\thisrow{1}), y expr=\thisrow{7} - \thisrow{11}]{\pOneFour};
\addplot table[x expr=(-21+\thisrow{1}), y index=7]{\pOneEight};
\addplot[plotcolor4,forget plot] table[x expr=(-21+\thisrow{1}), y expr=\thisrow{7} + \thisrow{11}]{\pOneEight};
\addplot[plotcolor4,forget plot] table[x expr=(-21+\thisrow{1}), y expr=\thisrow{7} - \thisrow{11}]{\pOneEight};
\addplot table[x expr=(-21+\thisrow{1}), y index=7]{\pOneSixteen};
\addplot[plotcolor5,forget plot] table[x expr=(-21+\thisrow{1}), y expr=\thisrow{7} + \thisrow{11}]{\pOneSixteen};
\addplot[plotcolor5,forget plot] table[x expr=(-21+\thisrow{1}), y expr=\thisrow{7} - \thisrow{11}]{\pOneSixteen};
\legend{\EA, (1+2)~EA, (1+4)~EA, (1+8)~EA, (1+16)~EA}
\end{axis}
\end{tikzpicture}
}
\subfigure[bit-wise noise]{
\begin{tikzpicture}[scale=0.76]
\begin{axis}[legend cell align=left, legend pos=south west, xlabel={$\log(q)$}, ymax=100,
forget plot style={opacity=0.35,mark=none},
]
\pgfplotstableread[header=false]{experiments/bestrealfitness2-(1+1)EA-n=100-bitwise-p-increasing-to-0.5-stopped-after-10n2-1000runs.txt}\pOneOne
\pgfplotstableread[header=false]{experiments/bestrealfitness2-(1+2)EA-n=100-bitwise-p-increasing-to-0.5-stopped-after-10n2-1000runs.txt}\pOneTwo
\pgfplotstableread[header=false]{experiments/bestrealfitness2-(1+4)EA-n=100-bitwise-p-increasing-to-0.5-stopped-after-10n2-1000runs.txt}\pOneFour
\pgfplotstableread[header=false]{experiments/bestrealfitness2-(1+8)EA-n=100-bitwise-p-increasing-to-0.5-stopped-after-10n2-1000runs.txt}\pOneEight
\pgfplotstableread[header=false]{experiments/bestrealfitness2-(1+16)EA-n=100-bitwise-p-increasing-to-0.5-stopped-after-10n2-1000runs.txt}\pOneSixteen
\addplot table[x expr=(-21+\thisrow{1}), y index=7]{\pOneOne};
\addplot[plotcolor1,forget plot] table[x expr=(-21+\thisrow{1}), y expr=\thisrow{7} + \thisrow{11}]{\pOneOne};
\addplot[plotcolor1,forget plot] table[x expr=(-21+\thisrow{1}), y expr=\thisrow{7} - \thisrow{11}]{\pOneOne};
\addplot table[x expr=(-21+\thisrow{1}), y index=7]{\pOneTwo};
\addplot[plotcolor2,forget plot] table[x expr=(-21+\thisrow{1}), y expr=\thisrow{7} + \thisrow{11}]{\pOneTwo};
\addplot[plotcolor2,forget plot] table[x expr=(-21+\thisrow{1}), y expr=\thisrow{7} - \thisrow{11}]{\pOneTwo};
\addplot table[x expr=(-21+\thisrow{1}), y index=7]{\pOneFour};
\addplot[plotcolor3,forget plot] table[x expr=(-21+\thisrow{1}), y expr=\thisrow{7} + \thisrow{11}]{\pOneFour};
\addplot[plotcolor3,forget plot] table[x expr=(-21+\thisrow{1}), y expr=\thisrow{7} - \thisrow{11}]{\pOneFour};
\addplot table[x expr=(-21+\thisrow{1}), y index=7]{\pOneEight};
\addplot[plotcolor4,forget plot] table[x expr=(-21+\thisrow{1}), y expr=\thisrow{7} + \thisrow{11}]{\pOneEight};
\addplot[plotcolor4,forget plot] table[x expr=(-21+\thisrow{1}), y expr=\thisrow{7} - \thisrow{11}]{\pOneEight};
\addplot table[x expr=(-21+\thisrow{1}), y index=7]{\pOneSixteen};
\addplot[plotcolor5,forget plot] table[x expr=(-21+\thisrow{1}), y expr=\thisrow{7} + \thisrow{11}]{\pOneSixteen};
\addplot[plotcolor5,forget plot] table[x expr=(-21+\thisrow{1}), y expr=\thisrow{7} - \thisrow{11}]{\pOneSixteen};
\legend{\EA, (1+2)~EA, (1+4)~EA, (1+8)~EA, (1+16)~EA}
\end{axis}
\end{tikzpicture}
}
\caption{Average best fitness during 1000 runs for the \lEA with $\lambda \in \{1, 2, 4, 8, 16\}$ on \LO ($n=100$) with one-bit prior noise with probability~$p \in \{2^{-20}, 2^{-19}, \dots, 2^{-1}\}$ and bit-wise prior noise $(1, q/n)$ with $q \in \{2^{-20}, 2^{-19}, \dots, 2^{0}\}$. Runs were stopped after $10n^2$ generations. Transparent lines show means $\pm$ standard deviation.}
\label{fig:average-best-fitness}
\end{figure*}

Another interesting performance measure not covered by our theoretical results is to inspect the best fitness found during a run before either finding an optimum or being stopped at $10n^2$ generations. Figure~\ref{fig:average-best-fitness} shows averages over these values. For the \EA the best fitness steadily decreases when increasing the noise parameter beyond the threshold for inefficient running times, reaching values of $30.414$ for one-bit noise with $p=1/2$ and $25.781$ for bit-wise noise with $q=1$. For comparison, the average best fitness found during $10n^2 = 10^5$ uniform random samples was $16.926$.
Again, we see that offspring populations help by shifting the curves towards higher noise strengths.

\section{An Example Where Noise Helps}
\label{sec:noise-helps}

%

The results so far show that on \LO, noise is disruptive and larger noise values lead to higher expected optimisation times.

The final contribution of this paper is to look at noise from a very different angle. We will show that noise can be beneficial for escaping from local optima. To this end, we consider a known class of functions that lead to a highly rugged fitness landscape with an underlying gradient pointing towards the location of the global optimum.
Such landscapes are known as ``big valley'' structures, which is an important characteristic of many
hard problems from combinatorial optimisation~\cite{Ochoa2016,Reeves1999}.

Pr{\"u}gel-Bennett defined such a class of problems known as \hurdle problems~\cite{PRUGELBENNETT2004135} as an example function where genetic algorithms with crossover outperform hill climbers. \hurdle functions are functions of unitation, that is, they only depend on the number of 1-bits. The fitness is given as
\[
\hurdle(x) = -\left\lceil \frac{\zeros{x}}{w} \right\rceil - \frac{\zeros{x} \bmod w}{w}
\]
where $\zeros{x}$ denotes the number of 0-bits in~$x$ and $w$ is a parameter called \emph{hurdle width} that defines the distance between subsequent peaks. A sketch of the function is shown in Figure~\ref{fig:hurdle}.

\begin{figure}[ht]
\centerline{
\begin{tikzpicture}[yscale=0.65, xscale=0.95]
\begin{axis}[xmin=0,xmax=20,ymax=0,ymin=-5-3/4,grid=both,xtick distance=4,xlabel={$\zeros{x}$}, ylabel={$\hurdle(x)$}]
\addplot+[blue] coordinates {
(0, 0) (1, -1-1/4) (2, -1-2/4) (3, -1-3/4) (4, -1) (5, -2-1/4) (6, -2-2/4) (7, -2-3/4) (8, -2) (9, -3-1/4) (10, -3-2/4) (11, -3-3/4) (12, -3) (13, -4-1/4) (14, -4-2/4) (15, -4-3/4) (16, -4) (17, -5-1/4) (18, -5-2/4) (19, -5-3/4) (20, -5)
};
\end{axis}
\end{tikzpicture}
}
\caption{Sketch of a \hurdle function with hurdle width~$w=4$ and problem size~$n=20$.}
\label{fig:hurdle}
\end{figure}
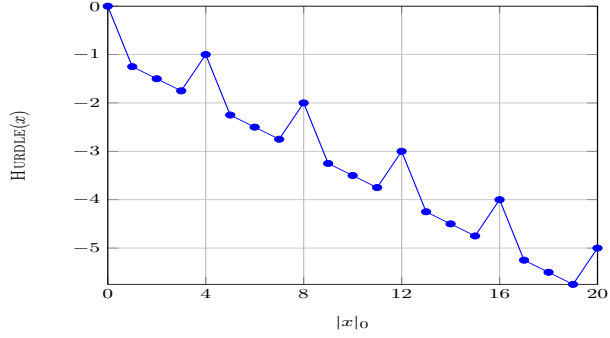

Here all search points with $i \bmod w = 0$ zeros are local optima, and all search points with $j$ zeros, $i - w < j < i$, have worse fitness. Hence an evolutionary algorithm needs to flip at least $w$ bits in order to find a search point of better fitness. \citet{Nguyen2018} proved that the \EA has expected time $\Theta(n^w)$ if $2 \le w \le n/2$.

In the following, we consider the well-known algorithm \emph{Randomised Local Search (RLS)}, which works like the \EA, but only flips exactly one bit in each mutation (chosen uniformly at random). We choose RLS instead of the \EA to keep the analyses simple and to make the point that even a very badly performing algorithm can be turned into a highly efficient algorithm through beneficial effects of noise. We will in particular show that RLS under noise is drastically faster than the \EA without noise. Section~\ref{sec:RLS-vs-EA} will further discuss whether results for RLS under noise can be transferred to the \EA under noise.

It is obvious that RLS has infinite expected time on any \hurdle function with non-trivial hurdle width $w \ge 2$, and \citet{Nguyen2018} showed via Chernoff bounds that local searchers get stuck in a non-optimal local optimum with probability $1-2^{-\Omega(n)}$ if $w \le (1 - \Omega(1))n/2$.

However, prior noise can help to escape from such a local optimum: RLS with one-bit prior noise can misevaluate either the parent or the offspring, which allows the algorithm to accept a search point with $i \bmod w = w-1$ ones. Then it can climb to the next local optimum from there, until the global optimum is found.
This is made precise in the following theorem.
\begin{theorem}
\label{the:noise-helps}
The expected optimisation time of RLS with one-bit prior noise~$p \le 1/(6n)$ on \hurdle with hurdle width~$w \ge 2\log n$ is $O(n^2/(pw^2) + n \log n)$.
\end{theorem}
Note that in particular for $p = 1/(6n)$ and $w = \Omega(n/\sqrt{\log n})$ this is $O(n \log n)$. Then RLS is as efficient as on the underlying function \textsc{OneMax} without any hurdles.
\begin{proofof}{Theorem~\ref{the:noise-helps}}
The algorithm can escape from a local optimum with $i$ zeros, $i \bmod w = 0$, if the offspring has $i-1$ zeros (probability $i/n$) and additionally
\begin{enumerate}
\item the offspring is misevaluated as having $i$ zeros (probability $p (n-i+1)/n$) or
\item the parent is misevaluated as having $i-1$ zeros (probability $pi/n$).
\end{enumerate}
The probability of the union of these events is
\[
\frac{p(n-i+1)}{n} + \frac{pi}{n} - \frac{p^2i(n-i+1)}{n^2} = p \left(1+\frac{1}{n}-\frac{pi(n-i+1)}{n^2}\right) \ge p\left(1+\frac{1}{n}-p\right) \ge p
\]
as the event of both offspring and parent being misevaluated as described is counted twice in the enumeration. Together, the probability of escaping from a local optimum with $i$ zeros is at least $pi/n$.

We now define a potential function $g$ such that $g(i)$ estimates or overestimates the expected optimisation time from a state with $i$ zeros, bar constant factors. Let $a_{i} := 2^{(i \bmod w)-w+1}$, then
\[
g(i) := \begin{cases}
0 & \text{if $i=0$},\\
g(i-1) + \frac{n}{ip} & \text{if $i > 0$, $i \bmod w = 0$,}\\
g(i-1) + \frac{n}{i} + a_i \frac{n^2}{i^2p(1-p)^2} & \text{otherwise.}
\end{cases}
\]
The term $a_i \frac{n^2}{i^2p(1-p)^2}$ is necessary since on a slope towards a local optimum there is a chance to increase the number of zeros and to possibly return to a worse, previously visited local optimum. The term is largest, $\frac{n^2}{i^2p(1-p)^2}$, for $i=w-1 \bmod w$ as from there returning to a local optimum with $i+1$ zeros is very likely. This needs to be accounted for in our choice of potential function. The term decreases exponentially for decreasing $i \bmod w$ since this risk is reduced as the algorithm moves away from a local optimum.

Note that $g(0) \le g(1) \le \dots \le g(n)$, with $g(n)$ being composed of the following sums.
The additive terms $\frac{n}{i}$ for all $i > 0, i \bmod w > 0$ sum up to at most $\sum_{i=1}^n \frac{n}{i} = O(n \log n)$. For each hurdle with a peak at $i$ zeros, $g(n)$ contains an additive term $\frac{n}{ip}$ as well as terms
\[
\sum_{j=1}^{w-1} 2^{j-w+1} \frac{n^2}{(i-w+j)^2p(1-p)^2}
\le O(1) \cdot \frac{n^2}{i^2p(1-p)^2}
\]
as $\sum_{d=0}^{i-1} 2^{-d} i^2/(i-d)^2 = O(1)$.
Adding up the terms for each hurdle with $w, 2w, 3w, \dots, (n/w)w$ zeros yields
\begin{align*}
g(i) \le g(n) =\;&
O\bigg(n \log n + \sum_{j=1}^{n/w} \bigg(\frac{n}{jwp} + \frac{n^2}{(jw)^2p(1-p)^2}\bigg)\bigg)\\
=\;&
O\bigg(n \log n + \frac{n}{wp} \sum_{j=1}^{n/w} \frac{1}{j} +
\frac{n^2}{w^2p(1-p)^2} \sum_{j=1}^{n/w}\frac{1}{j^2}\bigg)\\
=\;&
O\left(n \log n + \frac{n \log(n/w)}{wp} +
\frac{n^2}{w^2p}\right)\\
=\;&
O\left(n \log n +
\frac{n^2}{w^2p}\right)
\end{align*}
where the penultimate line follows from $\sum_{j=1}^{n/w} 1/j^2 \le \sum_{j=1}^\infty 1/j^2 = \pi^2/6 = O(1)$ and in the last line we used $\log(n/w) = O(n/w)$ to absorb the middle term.
We show in the following that the potential decreases in expectation by $\Omega(1)$.

For $0 < i \bmod w < w-1$, the potential decreases by ${g(i)-g(i-1)}$ if mutation creates a search point with $i-1$ zeros and the mutant is evaluated correctly (probability at least $i/n \cdot (1-p)$). It is increased by $g(i+1)-g(i)$ only if mutation creates a search point with $i+1$ zeros (probability $(n-i)/n \le 1$) and either the parent or the offspring is misevaluated (probability at most $2p$), as otherwise the offspring will be rejected.
Thus for all $i$ with $i \bmod w \notin \{0, w-1\}$, using $a_{i+1} = 2a_i$,
\begin{align*}
& E(g(X_t) - g(X_{t+1}) \mid X_t = i, i \bmod w \notin \{0, w-1\})\\
\ge\;& \frac{i}{n} (1-p)(g(i) - g(i-1)) - 2p (g(i+1)-g(i))\\
=\;& \frac{i}{n} (1-p)\left(\frac{n}{i} + \frac{a_i n^2}{i^2p(1-p)^2}\right) - 2p \left(\frac{n}{i+1} + \frac{a_{i+1} n^2}{(i+1)^2p(1-p)^2}\right)\\
\ge\;& 1-p + (1-p) \frac{a_i n}{ip(1-p)^2} - 2p \left(\frac{n}{i} + \frac{2a_{i} n^2}{i^2p(1-p)^2}\right)\\
=\;& 1-p - \frac{2pn}{i} + \frac{a_in}{ip(1-p)^2} \left(1-p - \frac{4pn}{i}\right).
\end{align*}
As $p \le 1/(6n)$, the bracket is at least $1-1/(6n) - 2/3 \ge 0$, hence the drift is at least
\begin{align*}
& E(g(X_t) - g(X_{t+1}) \mid X_t = i, i \bmod w \notin \{0, w-1\})\\
\ge\;& 1-p - \frac{2pn}{i} \ge 1 - \frac{1}{6n} - \frac{1}{3} \ge \frac{1}{2}.
\end{align*}
For $i \bmod w = 0$, the potential is decreased by $g(i) - g(i-1) = \frac{n}{ip}$ with probability at least $pi/n$, and it is increased by $g(i+1)-g(i)$ only if either the parent or the offspring is misevaluated and the offspring increases the number of zeros. The probability of an increase is bounded by~$2p$. Thus
\begin{align*}
& E(g(X_t) - g(X_{t+1}) \mid X_t = i, i \bmod w = 0)\\
\ge\;& \frac{n}{ip} \cdot \frac{ip}{n} - 2p (g(i+1)-g(i))\\
=\;& 1 - 2p (g(i+1)-g(i))\\
=\;& 1 - 2p \cdot \left(\frac{n}{i+1} + 2^{-w+2} \cdot \frac{n^2}{(i+1)^2p(1-p)^2}\right)\\
\ge\;& 1 - 2pn - 2^{-w+3} \cdot \frac{n^2}{i^2(1-p)^2}\\
\intertext{and using $p \le 1/(6n)$, $i \ge w$ and $w \ge 2\log n$ this is at least}
\ge\;& \frac{2}{3} - \frac{8}{w^2(1-p)^2}
\ge \frac{2}{3} - o(1).
\end{align*}
For $i \bmod w = w-1$ the potential is decreased by ${g(i)-g(i-1)}$ if mutation decreases the number of zeros and both parent and offspring are evaluated truthfully. The potential is increased by ${g(i+1)-g(i)}$ only if mutation creates a search point with $i+1$ zeros (probability at most~1). Thus
\begin{align*}
& E(g(X_t) - g(X_{t+1}) \mid X_t = i, i \bmod w = w-1)\\
\ge\;& \frac{i(1-p)^2}{n} \cdot (g(i)-g(i-1)) - (g(i+1)-g(i))\\
=\;& \frac{i(1-p)^2}{n} \cdot \left(\frac{n}{i} + \frac{n^2}{i^2p(1-p)^2}\right) - \frac{n}{(i+1)p}\\
=\;& (1-p)^2 + \frac{n}{ip} - \frac{n}{(i+1)p}\\
\ge\;& (1-p)^2 = 1 - O(1/n).
\end{align*}
For all states $i > 0$, the expected decrease in $g(X_t)$ is at least $c$ for a suitable constant~$c > 0$. Once $g(X_t) = 0$ is reached, an optimum is found. Standard additive drift analysis (see, e.\,g.~\cite[Theorem~1]{Lehre2018} for a self-contained statement and proof) then implies that the expected time until $g(X_t) = 0$ is reached is at most $g(n)/c = O(g(n)) = O(n \log n + n^2/(w^2p))$.
\end{proofof}

The reason why prior noise is helpful is that, intuitively speaking, it can ``smooth out'' the fitness landscape, blurring rugged peaks and allowing the algorithm to see the underlying gradient.
Hence noise can be useful for problems with a \emph{big valley} structure~\cite{Ochoa2016,Reeves1999}.
This effect has been observed in continuous spaces before~\cite{Rana1996} where it was termed ``annealing of peaks''. In discrete spaces the only other examples the author is aware of showing a positive effect of noise are
deceptive functions and needle-in-a-haystack functions~\cite{Qian2016}.

To put our result in perspective, we have shown that noise can mitigate a poor choice of algorithm. In our case, an elitist algorithm became a non-elitist algorithm because of noise. This is helpful for \hurdle as here non-elitism is advantageous, while even a small amount of non-elitism is clearly detrimental for \LO. Note that, as argued in~\citep[Section~4]{Akimoto2015}, noise can never improve an \emph{optimal} algorithm for a particular problem. If noise was able to improve the performance of an optimal algorithm, we could simply simulate the effect of noise in the algorithm and obtain a better performing algorithm.

%

\subsection{Experiments}

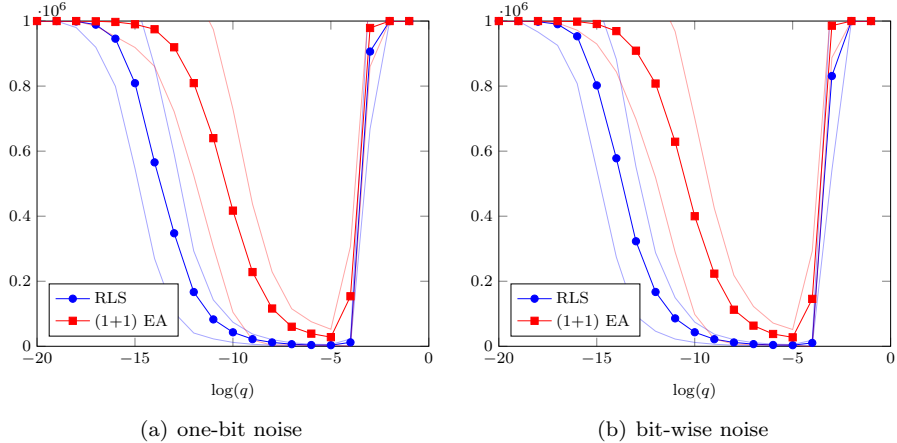
\begin{figure*}[bt]
\centering
\subfigure[one-bit noise]{
\begin{tikzpicture}[scale=0.76]
\begin{axis}[legend cell align=left, legend pos=south west, xlabel={$\log(q)$}, ymin=0, ymax=1000000,
xmin=-20,xmax=0,
forget plot style={opacity=0.35,mark=none},
]
\pgfplotstableread[header=false]{experiments/runtime-RLS-hurdle-14-n=100-onebit-p-increasing-to-0.5-stopped-after-n3-1000runs.txt}\pRLS
\pgfplotstableread[header=false]{experiments/runtime-(1+1)EA-hurdle-14-n=100-onebit-p-increasing-to-0.5-stopped-after-n3-1000runs.txt}\pEA
\addplot table[x expr=(-21+\thisrow{1}), y index=7]{\pRLS};
\addplot[plotcolor1,forget plot] table[x expr=(-21+\thisrow{1}), y expr=\thisrow{7} + \thisrow{11}]{\pRLS};
\addplot[plotcolor1,forget plot] table[x expr=(-21+\thisrow{1}), y expr=\thisrow{7} - \thisrow{11}]{\pRLS};
\addplot table[x expr=(-21+\thisrow{1}), y index=7]{\pEA};
\addplot[plotcolor2,forget plot] table[x expr=(-21+\thisrow{1}), y expr=\thisrow{7} + \thisrow{11}]{\pEA};
\addplot[plotcolor2,forget plot] table[x expr=(-21+\thisrow{1}), y expr=\thisrow{7} - \thisrow{11}]{\pEA};
\legend{RLS, \EA}
\end{axis}
\end{tikzpicture}
}
\subfigure[bit-wise noise]{
\begin{tikzpicture}[scale=0.76]
\begin{axis}[legend cell align=left, legend pos=south west, xlabel={$\log(q)$}, ymin=0, ymax=1000000,
xmin=-20,xmax=0,
forget plot style={opacity=0.35,mark=none},
]
\pgfplotstableread[header=false]{experiments/runtime-(1+1)EA-hurdle-14-n=100-bitwise-p-increasing-to-0.5-stopped-after-n3-1000runs.txt}\pEA
\pgfplotstableread[header=false]{experiments/runtime-RLS-hurdle-14-n=100-bitwise-p-increasing-to-0.5-stopped-after-n3-1000runs.txt}\pRLS
\addplot table[x expr=(-21+\thisrow{1}), y index=7]{\pRLS};
\addplot[plotcolor1,forget plot] table[x expr=(-21+\thisrow{1}), y expr=\thisrow{7} + \thisrow{11}]{\pRLS};
\addplot[plotcolor1,forget plot] table[x expr=(-21+\thisrow{1}), y expr=\thisrow{7} - \thisrow{11}]{\pRLS};
\addplot table[x expr=(-21+\thisrow{1}), y index=7]{\pEA};
\addplot[plotcolor2,forget plot] table[x expr=(-21+\thisrow{1}), y expr=\thisrow{7} + \thisrow{11}]{\pEA};
\addplot[plotcolor2,forget plot] table[x expr=(-21+\thisrow{1}), y expr=\thisrow{7} - \thisrow{11}]{\pEA};
\legend{RLS, \EA}
\end{axis}
\end{tikzpicture}
}
\caption{Average optimisation times during 1000 runs for RLS and the \EA on \hurdle with $n=100$ and hurdle width $w=14$ with one-bit prior noise with probability $p \in \{2^{-20}, 2^{-19}, \dots, 2^{-1}\}$ and bit-wise prior noise $(1, q/n)$ with $q \in \{2^{-20}, 2^{-19}, \dots, 2^{0}\}$. Runs were stopped after $10^6$ generations. Transparent lines show means $\pm$ standard deviation.}
\label{fig:hurdle-14}
\end{figure*}

We also provide experiments for \hurdle to see how well the theory predicts the average optimisation time, and to answer questions not covered by Theorem~\ref{the:noise-helps}.

Figure~\ref{fig:hurdle-14} shows the expected optimisation time of RLS and the \EA, for \hurdle with $n=100$ bits and a hurdle width of $w=\lceil 2\log n \rceil = 14$.
Runs were stopped after $n^3 = 10^6$ generations or when the optimum was found.
For one-bit noise with noise strength~$p$, the plots show that the algorithm is very efficient in the region $p \in \{2^{-10}, \dots, 2^{-4}\} \approx \{1/(10n), \dots, 6.4/n\}$ as predicted by Theorem~\ref{the:noise-helps}. The time further seems to increase with $1/p$ as $p$ is decreased, which matches the term $n^2/(pw^2)$ in the running time bound.

We can further see that as $p$ becomes too large, i.\,e., for $p \ge 2^{-3}$, the average time increases sharply. This matches known results for \onemax where $p=\omega((\log n)/n)$ leads to superpolynomial expected times~\cite{Droste2004}.

Figure~\ref{fig:hurdle-14} further shows that the choice of the noise model is insignificant: the results are nearly identical for one-bit prior noise~$p$ and bit-wise prior noise $(1, q/n)$ across all values of $p=q$.

\subsection{On the Performance of the \EA}
\label{sec:RLS-vs-EA}

The \EA shows a similar behaviour to RLS, except that there is a smaller window of efficient parameter ranges. The reader may think that Theorem~\ref{the:noise-helps} could also be proven for the \EA with a more complicated proof that considers all transition probabilities.

However, this is not the case. The problem for the \EA is that, compared to RLS, it is much more prone to climbing back up into the previous local optimum after making a fitness-decreasing jump towards the optimum. For instance, if $w=O(1)$ then there is always a constant probability of jumping to a local optimum with $w$ zeros from any search point with $1 \le i < w$ zeros. And the probability of moving close to the optimum is only of order $O(1/n)$, thus the conditional probability of moving closer to the global optimum in a generation where the \EA either moves closer or jumps to a state with $w$ zeros is still only $O(1/n)$. The algorithm may need to make several such steps in order to arrive at the optimum, and it loses all progress made if a jump back to state~$w$ occurs. This problem becomes less and less important as $w$ increases.

Note that the same fundamental challenge also exists for RLS as it can also move back to the previous local optimum. However, it can only increase the number of zeros by 1 in any step, and if the number of zeros is less than $w-1 \bmod w$, such a move will decrease the fitness and thus will only be accepted if noise makes the offspring appear competitive to the parent. In Theorem~\ref{the:noise-helps} the noise probability~$p$ is chosen low enough such that the latter is unlikely.

In the experiments from Figure~\ref{fig:hurdle-14}, the hurdle width $w=14$ is quite large in relation to the problem size $n=100$, so that the above issue does not affect performance too much. Decreasing the hurdle width shows a different picture: Figure~\ref{fig:hurdle-6} shows the performance of both algorithms for a smaller hurdle width of $w=6$ under one-bit noise.

\begin{figure*}[bt]
\centering
\begin{tikzpicture}[scale=0.76]
\begin{axis}[legend cell align=left, legend pos=south west, xlabel={$\log(p)$}, ymin=0, ymax=1000000,
xmin=-20,xmax=0,
forget plot style={opacity=0.35,mark=none},
]
\pgfplotstableread[header=false]{experiments/runtime-RLS-hurdle-6-n=100-onebit-p-increasing-to-0.5-stopped-after-n3-1000runs.txt}\pRLS
\pgfplotstableread[header=false]{experiments/runtime-(1+1)EA-hurdle-6-n=100-onebit-p-increasing-to-0.5-stopped-after-n3-1000runs.txt}\pEA
\addplot table[x expr=(-21+\thisrow{1}), y index=7]{\pRLS};
\addplot[plotcolor1,forget plot] table[x expr=(-21+\thisrow{1}), y expr=\thisrow{7} + \thisrow{11}]{\pRLS};
\addplot[plotcolor1,forget plot] table[x expr=(-21+\thisrow{1}), y expr=\thisrow{7} - \thisrow{11}]{\pRLS};
\addplot table[x expr=(-21+\thisrow{1}), y index=7]{\pEA};
\addplot[plotcolor2,forget plot] table[x expr=(-21+\thisrow{1}), y expr=\thisrow{7} + \thisrow{11}]{\pEA};
\addplot[plotcolor2,forget plot] table[x expr=(-21+\thisrow{1}), y expr=\thisrow{7} - \thisrow{11}]{\pEA};
\legend{RLS, \EA}
\end{axis}
\end{tikzpicture}
\caption{Average optimisation times during 1000 runs for RLS and the \EA on \hurdle with $n=100$ and hurdle width $w=6$ with one-bit prior noise with probability $p \in \{2^{-20}, 2^{-19}, \dots, 2^{-1}\}$. Runs were stopped after $10^6$ generations. Transparent lines show means $\pm$ standard deviation. The \EA failed in all runs, except for a single run at $\log(p)=-7$ that succeeded after 519377 generations.}
\label{fig:hurdle-6}
\end{figure*}
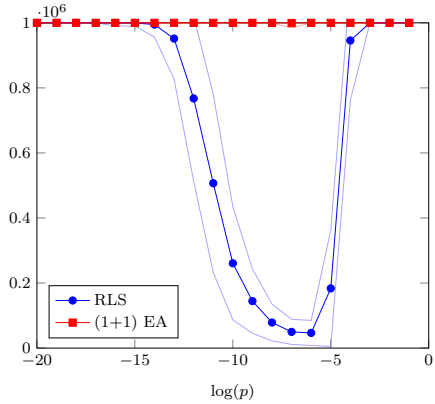

While RLS is still effective in the regime $p \in \{2^{-10}, \dots, 2^{-5}\}$ (even though the hurdle width is lower than required by Theorem~\ref{the:noise-helps}), the \EA failed in all runs, except for a single run at $\log(p)=-7$ that succeeded after 519{,}377 generations.

This conclusively shows why Theorem~\ref{the:noise-helps} had to be limited to RLS. As an aside, we have obtained a rare case where the performance of the \EA is drastically worse than that of RLS. So far, only very artificial examples were known~\cite{Doerr2008} and some of them, examples of \emph{monotone functions}, needed a significantly higher mutation rate~\cite{monotone-journal,lengler_steger_2018}.

\subsection{Offspring Populations are Harmful for Hurdle}

Finally, we consider the role of offspring populations on \hurdle, defining the (1+$\lambda$)~RLS as a variant of the \lEA where mutation flips exactly one bit. For consistency we refer to RLS as (1+1)~RLS.

The proof of Theorem~\ref{the:noise-helps} relies on the fact that a fitness-decreasing step leaving a local optimum towards the global optimum is accepted because of noise. While this effect was helpful on \LO, it is detrimental for \hurdle. This is shown empirically in Figure~\ref{fig:hurdle-offspring}.

\begin{figure*}[bt]
\centering
\begin{tikzpicture}[scale=0.8]
\begin{axis}[legend cell align=left, legend pos=outer north east, xlabel={$\log(p)$}, ymin=0, ymax=1000000,
xmin=-20,xmax=0,
ylabel={time},
forget plot style={opacity=0.35,mark=none},
]
\pgfplotstableread[header=false]{experiments/runtime-RLS-hurdle-14-n=100-onebit-p-increasing-to-0.5-stopped-after-n3-1000runs.txt}\pRLS
\pgfplotstableread[header=false]{experiments/runtime-(1+2)RLS-hurdle-14-n=100-bitwise-p-increasing-to-0.5-stopped-after-n3-1000runs.txt}\pRLSTwo
\pgfplotstableread[header=false]{experiments/runtime-(1+4)RLS-hurdle-14-n=100-bitwise-p-increasing-to-0.5-stopped-after-n3-1000runs.txt}\pRLSFour
\pgfplotstableread[header=false]{experiments/runtime-(1+8)RLS-hurdle-14-n=100-onebit-p-increasing-to-0.5-stopped-after-n3-1000runs.txt}\pRLSEight
\pgfplotstableread[header=false]{experiments/runtime-(1+16)RLS-hurdle-14-n=100-onebit-p-increasing-to-0.5-stopped-after-n3-1000runs.txt}\pRLSSixteen
\addplot table[x expr=(-21+\thisrow{1}), y index=7]{\pRLS};
\addplot[plotcolor1,forget plot] table[x expr=(-21+\thisrow{1}), y expr=\thisrow{7} + \thisrow{11}]{\pRLS};
\addplot[plotcolor1,forget plot] table[x expr=(-21+\thisrow{1}), y expr=\thisrow{7} - \thisrow{11}]{\pRLS};
\addplot table[x expr=(-21+\thisrow{1}), y index=7]{\pRLSTwo};
\addplot[plotcolor2,forget plot] table[x expr=(-21+\thisrow{1}), y expr=\thisrow{7} + \thisrow{11}]{\pRLSTwo};
\addplot[plotcolor2,forget plot] table[x expr=(-21+\thisrow{1}), y expr=\thisrow{7} - \thisrow{11}]{\pRLSTwo};
\addplot table[x expr=(-21+\thisrow{1}), y index=7]{\pRLSFour};
\addplot[plotcolor3,forget plot] table[x expr=(-21+\thisrow{1}), y expr=\thisrow{7} + \thisrow{11}]{\pRLSFour};
\addplot[plotcolor3,forget plot] table[x expr=(-21+\thisrow{1}), y expr=\thisrow{7} - \thisrow{11}]{\pRLSFour};
\addplot table[x expr=(-21+\thisrow{1}), y index=7]{\pRLSEight};
\addplot[plotcolor4,forget plot] table[x expr=(-21+\thisrow{1}), y expr=\thisrow{7} + \thisrow{11}]{\pRLSEight};
\addplot[plotcolor4,forget plot] table[x expr=(-21+\thisrow{1}), y expr=\thisrow{7} - \thisrow{11}]{\pRLSEight};
\addplot table[x expr=(-21+\thisrow{1}), y index=7]{\pRLSSixteen};
\addplot[plotcolor5,forget plot] table[x expr=(-21+\thisrow{1}), y expr=\thisrow{7} + \thisrow{11}]{\pRLSSixteen};
\addplot[plotcolor5,forget plot] table[x expr=(-21+\thisrow{1}), y expr=\thisrow{7} - \thisrow{11}]{\pRLSSixteen};
\legend{(1+1)~RLS, (1+2)~RLS, (1+4)~RLS, (1+8)~RLS, (1+16)~RLS}
\end{axis}
\end{tikzpicture}
\caption{Average optimisation times during 1000 runs for (1+$\lambda$)~RLS on \hurdle with $n=100$ and hurdle width $w=14$ with one-bit prior noise with probability $p \in \{2^{-20}, 2^{-19}, \dots, 2^{-1}\}$. Runs were stopped after $10^6$ generations. Transparent lines show means $\pm$ standard deviation.}
\label{fig:hurdle-offspring}
\end{figure*}
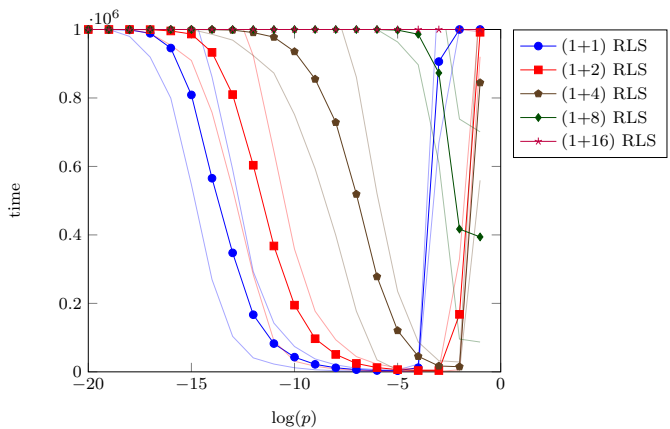

An increased offspring population shifts the curves towards higher noise parameters, while maintaining the unimodal shape of the curve, with steep increases for too large values. This shift is very similar to the one observed for the \lEA on \LO.

For instance, for $p \in \{2^{-10}, \dots, 2^{-4}\}$ where (1+1)~RLS is efficient, the (1+8)~RLS fails to find the optimum before time runs out in almost all runs, and the (1+16)~RLS only found the optimum in a single run at
$p=0.5$, with a time of 795{,}151 generations.
We conclude that, in this context, offspring populations can be harmful.

\section{Conclusions}

We have presented a simple method for proving upper bounds under several prior noise models, based on estimating the probability that during the median worst-case optimisation time no noise occurs. Despite its simplicity, it matches and generalises the best known results~\cite{Qian2018,Bian2018} and provides a unified approach for one-bit noise, bit-wise noise, and asymmetric bit-wise noise.
Along with our negative result for \LO, the expected optimisation time of the \EA on \LO is $\Theta(n^2) \cdot \exp(\Theta(\min\{p n^2, n\}))$ for one-bit noise~$p \le 1/2$, asymmetric one-bit noise with $p = O(1/n)$, and bit-wise noise $(p', q/n)$ where $q/n \le 1/2$ and $p = p'\min\{q, 1\}$. This confirms that the threshold between polynomial and superpolynomial expected times is $p = \Theta((\log n)/n^2)$ and $p = \Omega(1/n)$ leads to exponential expected times.

Offspring populations can cope with noise up to $p \le 1/2$ if the population size is at least $\lambda \ge
\log_{\frac{e}{e-1/2}}(n) \approx 3.42 \log n$.
We obtained an upper bound of
$O\big(n^2 \cdot e^{O(pn/\lambda)}\big)$, guaranteeing
polynomial expected times for $p = O((\lambda \log n)/n)$. An open problem is whether the upper bound is tight in the same sense as for the \EA.

Finally, we showed that on the \hurdle problem class, a highly rugged problem with a clear ``big valley'' structure, prior noise is helpful as it allows RLS to escape from local optima and to follow the underlying gradient.
Experiments complemented our theoretical results and also showed that RLS under noise outperforms the \EA both with and without noise. Experiments further showed that on \hurdle, in stark contrast to \LO, offspring populations in RLS can be harmful as here they reduce the beneficial effects of noise.

Open problems for future work include showing a lower bound for the expected optimisation time of the \lEA on \LO, and obtaining tighter results on the performance of evolutionary algorithms with parent populations, i.\,e., the \muea, on \LO and other problems.

\subsubsection*{Acknowledgements}
The author thanks the anonymous GECCO reviewers for their many thorough and constructive comments that helped to improve the preliminary version~\cite{Sudholt2018a}. Many thanks to Chao Bian and Chao Qian for pointing out mistakes in the proofs of Lemma~\ref{lem:fallback} and Lemma~\ref{lem:recovery-after-fallback} in~\cite{Sudholt2018a} and to Benjamin Doerr for insightful discussions.

\bibliographystyle{abbrvnat}

\begin{thebibliography}{51}
\providecommand{\natexlab}[1]{#1}
\providecommand{\url}[1]{\texttt{#1}}
\expandafter\ifx\csname urlstyle\endcsname\relax
  \providecommand{\doi}[1]{doi: #1}\else
  \providecommand{\doi}{doi: \begingroup \urlstyle{rm}\Url}\fi

\bibitem[Akimoto et~al.(2015)Akimoto, Astete-Morales, and Teytaud]{Akimoto2015}
Y.~Akimoto, S.~Astete-Morales, and O.~Teytaud.
\newblock Analysis of runtime of optimization algorithms for noisy functions
  over discrete codomains.
\newblock \emph{Theoretical Computer Science}, 605:\penalty0 42--50, 2015.

\bibitem[Badkobeh et~al.(2015)Badkobeh, Lehre, and Sudholt]{Badkobeh2015}
G.~Badkobeh, P.~K. Lehre, and D.~Sudholt.
\newblock Black-box complexity of parallel search with distributed populations.
\newblock In \emph{Proceedings of Foundations of Genetic Algorithms (FOGA
  '15)}, pages 3--15. ACM Press, 2015.

\bibitem[Baswana et~al.(2009)Baswana, Biswas, Doerr, Friedrich, Kurur, and
  Neumann]{Baswana2009}
S.~Baswana, S.~Biswas, B.~Doerr, T.~Friedrich, P.~P. Kurur, and F.~Neumann.
\newblock Computing single source shortest paths using single-objective fitness
  functions.
\newblock In \emph{Proceedings of FOGA~'09}, pages 59--66. ACM Press, 2009.

\bibitem[Beyer(2000)]{BEYER2000239}
H.-G. Beyer.
\newblock Evolutionary algorithms in noisy environments: theoretical issues and
  guidelines for practice.
\newblock \emph{Computer Methods in Applied Mechanics and Engineering},
  186\penalty0 (2):\penalty0 239 -- 267, 2000.

\bibitem[Bian et~al.(2018)Bian, Qian, and Tang]{Bian2018}
C.~Bian, C.~Qian, and K.~Tang.
\newblock Towards a running time analysis of the {(1+1)-EA} for {OneMax} and
  {LeadingOnes} under general bit-wise noise.
\newblock In A.~Auger, C.~M. Fonseca, N.~Louren{\c{c}}o, P.~Machado,
  L.~Paquete, and D.~Whitley, editors, \emph{Parallel Problem Solving from
  Nature -- PPSN XV}, pages 165--177, Cham, 2018. Springer International
  Publishing.

\bibitem[Bianchi et~al.(2009)Bianchi, Dorigo, Gambardella, and
  Gutjahr]{Bianchi2009}
L.~Bianchi, M.~Dorigo, L.~M. Gambardella, and W.~J. Gutjahr.
\newblock A survey on metaheuristics for stochastic combinatorial optimization.
\newblock \emph{Natural Computing}, 8\penalty0 (2):\penalty0 239--287, 2009.

\bibitem[Covantes~Osuna et~al.(2017)Covantes~Osuna, Gao, Neumann, and
  Sudholt]{CovantesOsuna2017a}
E.~Covantes~Osuna, W.~Gao, F.~Neumann, and D.~Sudholt.
\newblock Speeding up evolutionary multi-objective optimisation through
  diversity-based parent selection.
\newblock In \emph{Proceedings of the Genetic and Evolutionary Computation
  Conference (GECCO~'17)}, pages 553--560. ACM, 2017.

\bibitem[Cutello et~al.(2007)Cutello, Nicosia, and Pavone]{CutelloJCO}
V.~Cutello, G.~Nicosia, and M.~Pavone.
\newblock An immune algorithm with stochastic aging and kullback entropy for
  the chromatic number problem.
\newblock \emph{Journal of Combinatorial Optimization}, 14:\penalty0 9--33,
  2007.

\bibitem[Dang and Lehre(2015)]{Dang2015}
D.-C. Dang and P.~K. Lehre.
\newblock Efficient optimisation of noisy fitness functions with
  population-based evolutionary algorithms.
\newblock In \emph{Proceedings of Foundations of Genetic Algorithms (FOGA
  '15)}, pages 62--68. ACM, 2015.

\bibitem[Dang and Lehre(2016)]{Dang2016}
D.-C. Dang and P.~K. Lehre.
\newblock Runtime analysis of non-elitist populations: From classical
  optimisation to partial information.
\newblock \emph{Algorithmica}, 75\penalty0 (3):\penalty0 428--461, 2016.

\bibitem[Dang-Nhu et~al.(2018)Dang-Nhu, Dardinier, Doerr, Izacard, and
  Nogneng]{Dang-Nhu2018}
R.~Dang-Nhu, T.~Dardinier, B.~Doerr, G.~Izacard, and D.~Nogneng.
\newblock A new analysis method for evolutionary optimization of dynamic and
  noisy objective functions.
\newblock In \emph{Proceedings of the Genetic and Evolutionary Computation
  Conference (GECCO~'18)}, pages 1467--1474. ACM, 2018.

\bibitem[Doerr et~al.(2008)Doerr, Jansen, and Klein]{Doerr2008}
B.~Doerr, T.~Jansen, and C.~Klein.
\newblock Comparing global and local mutations on bit strings.
\newblock In \emph{Proceedings of the Genetic and Evolutionary Computation
  Conference (GECCO~'08)}, pages 929--936. ACM, 2008.

\bibitem[Doerr et~al.(2012)Doerr, Hota, and K\"{o}tzing]{Doe-Hot-Koe:c:12}
B.~Doerr, A.~Hota, and T.~K\"{o}tzing.
\newblock Ants easily solve stochastic shortest path problems.
\newblock In \emph{Proceedings of the Genetic and Evolutionary Computation
  Conference (GECCO '12)}, pages 17--24, 2012.

\bibitem[Doerr et~al.(2013{\natexlab{a}})Doerr, Jansen, Sudholt, Winzen, and
  Zarges]{monotone-journal}
B.~Doerr, T.~Jansen, D.~Sudholt, C.~Winzen, and C.~Zarges.
\newblock Mutation rate matters even when optimizing monotonic functions.
\newblock \emph{Evolutionary Computation}, 21\penalty0 (1):\penalty0 1--21,
  2013{\natexlab{a}}.

\bibitem[Doerr et~al.(2013{\natexlab{b}})Doerr, Sudholt, and Witt]{Doerr2013}
B.~Doerr, D.~Sudholt, and C.~Witt.
\newblock When do evolutionary algorithms optimize separable functions in
  parallel?
\newblock In \emph{Proceedings of Foundations of Genetic Algorithms
  (FOGA~'13)}, pages 51--64. ACM, 2013{\natexlab{b}}.

\bibitem[Droste(2004)]{Droste2004}
S.~Droste.
\newblock Analysis of the (1+1)~{EA} for a noisy {OneMax}.
\newblock In \emph{Proceedings of the Genetic and Evolutionary Computation
  Conference (GECCO 2004)}, pages 1088--1099. Springer, 2004.

\bibitem[Droste et~al.(2002)Droste, Jansen, and Wegener]{Droste2002}
S.~Droste, T.~Jansen, and I.~Wegener.
\newblock On the analysis of the (1+1) evolutionary algorithm.
\newblock \emph{Theoretical Computer Science}, 276\penalty0 (1--2):\penalty0
  51--81, 2002.

\bibitem[Feldmann and K\"{o}tzing(2013)]{Feldmann2013}
M.~Feldmann and T.~K\"{o}tzing.
\newblock Optimizing expected path lengths with ant colony optimization using
  fitness proportional update.
\newblock In \emph{Proceedings of Foundations of Genetic Algorithms (FOGA
  '13)}, pages 65--74. ACM, 2013.

\bibitem[Friedrich et~al.(2015)Friedrich, K\"{o}tzing, Krejca, and
  Sutton]{Friedrich2015}
T.~Friedrich, T.~K\"{o}tzing, M.~S. Krejca, and A.~M. Sutton.
\newblock Robustness of ant colony optimization to noise.
\newblock In \emph{Proceedings of the Genetic and Evolutionary Computation
  Conference (GECCO~'15)}, pages 17--24. ACM, 2015.

\bibitem[Friedrich et~al.(2017)Friedrich, K{\"o}tzing, Krejca, and
  Sutton]{Friedrich2017}
T.~Friedrich, T.~K{\"o}tzing, M.~S. Krejca, and A.~M. Sutton.
\newblock The compact genetic algorithm is efficient under extreme gaussian
  noise.
\newblock \emph{IEEE Transactions on Evolutionary Computation}, 21\penalty0
  (3):\penalty0 477--490, 2017.

\bibitem[Giel and Lehre(2010)]{Giel2010}
O.~Giel and P.~K. Lehre.
\newblock On the effect of populations in evolutionary multi-objective
  optimisation.
\newblock \emph{Evolutionary Computation}, 18\penalty0 (3):\penalty0 335--356,
  2010.

\bibitem[Gie{\ss}en and K{\"o}tzing(2016)]{Giessen2016}
C.~Gie{\ss}en and T.~K{\"o}tzing.
\newblock Robustness of populations in stochastic environments.
\newblock \emph{Algorithmica}, 75\penalty0 (3):\penalty0 462--489, 2016.

\bibitem[Jansen and Sudholt(2010)]{Jansen2010}
T.~Jansen and D.~Sudholt.
\newblock Analysis of an asymmetric mutation operator.
\newblock \emph{Evolutionary Computation}, 18\penalty0 (1):\penalty0 1--26,
  2010.

\bibitem[Jansen et~al.(2005)Jansen, De{\protect~}Jong, and
  Wegener]{Jansen2005a}
T.~Jansen, K.~A. De{\protect~}Jong, and I.~Wegener.
\newblock On the choice of the offspring population size in evolutionary
  algorithms.
\newblock \emph{Evolutionary Computation}, 13:\penalty0 413--440, 2005.

\bibitem[Jebalia et~al.(2011)Jebalia, Auger, and Hansen]{Jebalia2011}
M.~Jebalia, A.~Auger, and N.~Hansen.
\newblock Log-linear convergence and divergence of the scale-invariant
  {(1+1)-ES} in noisy environments.
\newblock \emph{Algorithmica}, 59\penalty0 (3):\penalty0 425--460, 2011.

\bibitem[Jin and Branke(2005)]{Jin2005}
Y.~Jin and J.~Branke.
\newblock Evolutionary optimization in uncertain environments-a survey.
\newblock \emph{IEEE Transactions on Evolutionary Computation}, 9\penalty0
  (3):\penalty0 303--317, 2005.

\bibitem[L{\"a}ssig and Sudholt(2013)]{Lassig2013}
J.~L{\"a}ssig and D.~Sudholt.
\newblock Design and analysis of migration in parallel evolutionary algorithms.
\newblock \emph{Soft Computing}, 17\penalty0 (7):\penalty0 1121--1144, 2013.

\bibitem[Laumanns et~al.(2004)Laumanns, Thiele, and Zitzler]{Laumanns2004}
M.~Laumanns, L.~Thiele, and E.~Zitzler.
\newblock Running time analysis of multiobjective evolutionary algorithms on
  pseudo-boolean functions.
\newblock \emph{{IEEE} Transactions on Evolutionary Computation}, 8\penalty0
  (2):\penalty0 170--182, 2004.

\bibitem[Lehre and Witt(2013)]{Lehre2018}
P.~K. Lehre and C.~Witt.
\newblock General drift analysis with tail bounds.
\newblock \emph{CoRR}, abs/1307.2559, 2013.
\newblock URL \url{http://arxiv.org/abs/1307.2559}.

\bibitem[Lengler and Steger(2018)]{lengler_steger_2018}
J.~Lengler and A.~Steger.
\newblock Drift analysis and evolutionary algorithms revisited.
\newblock \emph{Combinatorics, Probability and Computing}, 27\penalty0
  (4):\penalty0 643–666, 2018.

\bibitem[Levin et~al.(2008)Levin, Peres, and Wilmer]{Levin2008}
D.~A. Levin, Y.~Peres, and E.~L. Wilmer.
\newblock \emph{Markov Chains and Mixing Times}.
\newblock American Mathematical Society, 2008.

\bibitem[Meyer-Nieberg and Beyer(2008)]{Meyer-Nieberg2008}
S.~Meyer-Nieberg and H.-G. Beyer.
\newblock Why noise may be good: Additive noise on the sharp ridge.
\newblock In \emph{Proceedings of the Genetic and Evolutionary Computation
  Conference (GECCO '08)}, pages 511--518. ACM, 2008.

\bibitem[Mitzenmacher and Upfal(2005)]{Mitzenmacher2005}
M.~Mitzenmacher and E.~Upfal.
\newblock \emph{Probability and Computing}.
\newblock Cambridge University Press, 2005.

\bibitem[Nguyen et~al.(2015)Nguyen, Sutton, and Neumann]{NguyenSN15}
A.~Q. Nguyen, A.~M. Sutton, and F.~Neumann.
\newblock Population size matters: Rigorous runtime results for maximizing the
  hypervolume indicator.
\newblock \emph{Theoretical Computer Science}, 561:\penalty0 24--36, 2015.

\bibitem[Nguyen and Sudholt(2018)]{Nguyen2018}
P.~T.~H. Nguyen and D.~Sudholt.
\newblock Memetic algorithms beat evolutionary algorithms on the class of
  hurdle problems.
\newblock In \emph{Proceedings of the Genetic and Evolutionary Computation
  Conference (GECCO~'18)}, pages 1071--1078. ACM, 2018.

\bibitem[Ochoa and Veerapen(2016)]{Ochoa2016}
G.~Ochoa and N.~Veerapen.
\newblock Deconstructing the big valley search space hypothesis.
\newblock In \emph{Proceedings of Evolutionary Computation in Combinatorial
  Optimization (EvoCOP~2016)}, pages 58--73. Springer, 2016.

\bibitem[Oliveto and Sudholt(2014)]{Oliveto2014}
P.~S. Oliveto and D.~Sudholt.
\newblock On the runtime analysis of stochastic ageing mechanisms.
\newblock In \emph{Proceedings of the Genetic and Evolutionary Computation
  Conference (GECCO 2014)}, pages 113--120. ACM Press, 2014.

\bibitem[Paix{\~a}o et~al.(2017)Paix{\~a}o, P{\'e}rez~Heredia, Sudholt, and
  Trubenov{\'a}]{Paixao2016}
T.~Paix{\~a}o, J.~P{\'e}rez~Heredia, D.~Sudholt, and B.~Trubenov{\'a}.
\newblock Towards a runtime comparison of natural and artificial evolution.
\newblock \emph{Algorithmica}, 78\penalty0 (2):\penalty0 681--713, 2017.

\bibitem[Pr{\"u}gel-Bennett(2004)]{PRUGELBENNETT2004135}
A.~Pr{\"u}gel-Bennett.
\newblock When a genetic algorithm outperforms hill-climbing.
\newblock \emph{Theoretical Computer Science}, 320\penalty0 (1):\penalty0
  135--153, 2004.

\bibitem[Prugel-Bennett et~al.(2015)Prugel-Bennett, Rowe, and
  Shapiro]{Prugel-Bennett2015}
A.~Prugel-Bennett, J.~Rowe, and J.~Shapiro.
\newblock Run-time analysis of population-based evolutionary algorithm in noisy
  environments.
\newblock In \emph{Proceedings of Foundations of Genetic Algorithms
  (FOGA~'15)}, pages 69--75. ACM, 2015.

\bibitem[Qian et~al.(2013)Qian, Yu, and Zhou]{Qian2013}
C.~Qian, Y.~Yu, and Z.-H. Zhou.
\newblock An analysis on recombination in multi-objective evolutionary
  optimization.
\newblock \emph{Artificial Intelligence}, 204:\penalty0 99--119, 2013.

\bibitem[Qian et~al.(2018{\natexlab{a}})Qian, Bian, Jiang, and Tang]{Qian2018}
C.~Qian, C.~Bian, W.~Jiang, and K.~Tang.
\newblock Running time analysis of the {($1+1$)-EA} for {OneMax} and
  {LeadingOnes} under bit-wise noise.
\newblock \emph{Algorithmica}, 2018{\natexlab{a}}.

\bibitem[Qian et~al.(2018{\natexlab{b}})Qian, Bian, Yu, Tang, and
  Yao]{Qian2018b}
C.~Qian, C.~Bian, Y.~Yu, K.~Tang, and X.~Yao.
\newblock Analysis of noisy evolutionary optimization when sampling fails.
\newblock In \emph{Proceedings of the 20th ACM Conference on Genetic and
  Evolutionary Computation (GECCO~'18)}, pages 1507--1514, 2018{\natexlab{b}}.

\bibitem[Qian et~al.(2018{\natexlab{c}})Qian, Yu, Tang, Jin, Yao, and
  Zhou]{Qian2016a}
C.~Qian, Y.~Yu, K.~Tang, Y.~Jin, X.~Yao, and Z.-H. Zhou.
\newblock On the effectiveness of sampling for evolutionary optimization in
  noisy environments.
\newblock \emph{Evolutionary Computation}, 26\penalty0 (2):\penalty0 237--267,
  2018{\natexlab{c}}.

\bibitem[Qian et~al.(2018{\natexlab{d}})Qian, Yu, and Zhou]{Qian2016}
C.~Qian, Y.~Yu, and Z.-H. Zhou.
\newblock Analyzing evolutionary optimization in noisy environments.
\newblock \emph{Evolutionary Computation}, 26\penalty0 (1):\penalty0 1--41,
  2018{\natexlab{d}}.

\bibitem[Rana et~al.(1996)Rana, Whitley, and Cogswell]{Rana1996}
S.~Rana, L.~D. Whitley, and R.~Cogswell.
\newblock Searching in the presence of noise.
\newblock In \emph{Proceedings of PPSN IV}, pages 198--207. Springer, 1996.

\bibitem[Reeves(1999)]{Reeves1999}
C.~R. Reeves.
\newblock Landscapes, operators and heuristic search.
\newblock \emph{Annals of Operations Research}, 86\penalty0 (0):\penalty0
  473--490, 1999.

\bibitem[Rowe and Sudholt(2014)]{Rowe2013}
J.~E. Rowe and D.~Sudholt.
\newblock The choice of the offspring population size in the (1,$\lambda$)
  evolutionary algorithm.
\newblock \emph{Theoretical Computer Science}, 545:\penalty0 20--38, 2014.

\bibitem[Sudholt(2018)]{Sudholt2018a}
D.~Sudholt.
\newblock On the robustness of evolutionary algorithms to noise: Refined
  results and an example where noise helps.
\newblock In \emph{Proceedings of the Genetic and Evolutionary Computation
  Conference (GECCO 2018)}, pages 1523--1530. ACM, 2018.

\bibitem[Sudholt and Thyssen(2012{\natexlab{a}})]{Sudholt2011a}
D.~Sudholt and C.~Thyssen.
\newblock Running time analysis of ant colony optimization for shortest path
  problems.
\newblock \emph{Journal of Discrete Algorithms}, 10:\penalty0 165--180,
  2012{\natexlab{a}}.

\bibitem[Sudholt and Thyssen(2012{\natexlab{b}})]{Sudholt2012}
D.~Sudholt and C.~Thyssen.
\newblock A simple ant colony optimizer for stochastic shortest path problems.
\newblock \emph{Algorithmica}, 64\penalty0 (4):\penalty0 643--672,
  2012{\natexlab{b}}.

\end{thebibliography}

\end{document}